\documentclass{article}

\usepackage{wrapfig}

\usepackage[main, final]{neurips_2026}
\usepackage[utf8]{inputenc} 
\usepackage[T1]{fontenc}    
\usepackage{url}            
\usepackage{booktabs}       
\usepackage{amsfonts}       
\usepackage{nicefrac}       
\usepackage{microtype}      
\usepackage[table]{xcolor}  
\usepackage{hyperref}
\usepackage{graphicx}
\usepackage{subcaption}
\usepackage{enumitem}

\usepackage{algorithm}
\usepackage{algorithmic} 

\usepackage{multirow}

\usepackage{amsmath}
\usepackage{amssymb}
\usepackage{mathtools}
\usepackage{amsthm}

\usepackage[capitalize,noabbrev]{cleveref}
\usepackage{algorithm}
\usepackage{algorithmic}

\makeatletter
\renewcommand{\fs@ruled}{%
  \def\@fs@cfont{\bfseries}\let\@fs@capt\floatc@ruled
  \def\@fs@pre{\hrule height 1.2pt depth 0pt \kern 4pt}%
  \def\@fs@post{\kern 4pt\hrule height 1.2pt\relax}%
  \def\@fs@mid{\kern 4pt\hrule\kern 4pt}%
  \let\@fs@iftopcapt\iftrue}
\makeatother

\floatstyle{ruled}
\restylefloat{algorithm}

\definecolor{commentgray}{gray}{0.5}

\newcommand{\GrayComment}[1]{%
    \STATE \quad \textcolor{commentgray}{\small \textit{// #1}}%
}

\theoremstyle{plain}
\newtheorem{theorem}{Theorem}[section]
\newtheorem{proposition}[theorem]{Proposition}

\theoremstyle{definition}

\newtheorem{assumption}[theorem]{Assumption}
\theoremstyle{remark}
\newtheorem{remark}[theorem]{Remark}

\title{Escaping the Variance Trap: Jacobian-Free Dynamics for Root-Finding Bilevel Optimization}

\author{%
  Zhiyu Li \\
  University of Science and Technology of China \\
  Hefei, China \\
  \texttt{lzyblank@mail.ustc.edu.cn} \\
  \And
  Xi Xuan \\
  City University of Hong Kong \\
  Hong Kong SAR, China \\
  \texttt{xixuan3@cityu.edu.hk} \\
  \And
Davide Carbone \\
  Laboratoire de Physique de l'{\'E}cole Normale Sup{\'e}rieure, \\
  Universit{\'e} PSL, CNRS, Sorbonne Universit{\'e}, \\
  Universit{\'e} de Paris \\
  Paris, France \\
  \texttt{davide.carbone@phys.ens.fr} \\
}

\begin{document}

\maketitle
\vspace{-0.5cm}
\begin{abstract}
Many central machine learning tasks, from entropy tuning in reinforcement learning to equilibrating generative adversarial networks, are fundamentally stochastic root-finding problems rather than loss minimization. Yet, they are frequently forced into a minimization framework via squared residuals, introducing a critical flaw we identify as the Variance Trap. Standard bilevel minimization algorithms require estimating hypergradients involving implicit Jacobians; in stochastic settings, these terms act as noise amplifiers, destabilizing convergence. We formalize Root-Finding Bilevel Optimization (RF-BO) as a distinct problem class that bypasses this pathology. We propose a Jacobian-free solution using Two-Time-Scale Stochastic Approximation (TTSA) that updates directly along the root error, structurally avoiding variance amplification. We provide the first non-asymptotic convergence guarantees for TTSA in this setting under Markovian noise. Extensive experiments demonstrate the decisive advantage of this paradigm: compared to squared-residual and implicit-gradient baselines, our framework achieves a 2.6\% top-1 accuracy gain in SimCLR, 17$\times$ faster convergence in non-linear ODE control where baselines fail, significantly improved entropy stability in reinforcement learning, and an 11.1\% quality improvement in generative modeling.
\end{abstract}

\section{Introduction}

Bilevel optimization (BO) has become the standard mathematical language for hierarchical
learning tasks \citep{franceschi2018bilevel, liu2018darts}, with the canonical formulation
\begin{equation}
\min_{\alpha} F(\alpha, \theta^*(\alpha)), \quad \text{s.t.} \quad \theta^*(\alpha) =
\arg\min_\theta G(\theta; \alpha),
\end{equation}
where $\alpha \in \mathbb{R}^d$ and $\theta \in \mathbb{R}^p$ are the upper and lower-level
parameters. Solving this requires the hypergradient $\nabla_\alpha F$, involving the implicit
Jacobian $\nabla_\alpha \theta^*(\alpha) \approx -[\nabla_{\theta\theta}^2 G]^{-1}
\nabla_{\alpha\theta}^2 G$—a central bottleneck incurring prohibitive memory, computation,
and numerical instability \citep{lorraine2020optimizing, ji2021bilevel}.

However, a pervasive class of problems typically shoehorned into this framework are
naturally equilibrium-seeking or root-finding tasks, not minimization---e.g., SAC
temperature tuning aims to match a target entropy \citep{haarnoja2018soft}, and stabilizing
WGANs involves satisfying a gradient penalty constraint \citep{gulrajani2017improved}.
We term this class Root-Finding Bilevel Optimization (RF-BO):
\begin{equation}
\text{Find } \alpha \in \mathbb{R}^d \text{ s.t. } \mathbb{E}[h(\alpha, \theta^*(\alpha))] = 0,
\end{equation}
where $h: \mathbb{R}^d \times \mathbb{R}^p \to \mathbb{R}^d$ is a vector-valued residual
map encoding the equilibrium condition.
The Variance Trap of Minimization. A common heuristic reformulates RF-BO as minimizing
$\frac{1}{2}\|h\|^2$, but this is statistically flawed in stochastic regimes: the gradient
$(\nabla_\alpha h)^\top h$ couples the implicit Jacobian with the residual, amplifying
lower-level estimation noise by the Hessian condition number and residual magnitude---a
Variance Trap causing variance explosion. Unlike \citet{kwon2023penalty} and
\citet{hu2023contextual}, which grapple with these Jacobian-induced instabilities, our
RF-BO framework addresses the root cause by abandoning the minimization objective entirely.

Our Approach: Jacobian-Free Updates. We employ the classical Two-Time-Scale Stochastic
Approximation (TTSA) to update $\alpha$ directly using the residual map $h$, avoiding the
implicit Jacobian. Our core contribution lies in formalizing the root-finding problem class
and theoretically diagnosing the variance amplification in squared-residual minimization:
this Jacobian-free process naturally bounds the update variance by bypassing the Hessian
inverse (Proposition~\ref{prop:variance}). Despite our guarantees assuming PL conditions
for non-convex lower levels (with unique minimizers), experiments validate empirical
robustness in deep, multi-modal settings such as GAN training.

Our contributions are threefold. \textbf{First}, we formalize the RF-BO class and
theoretically characterize the Variance Trap in squared-residual methods, directly
addressing stability concerns in bi-level RL and LLM alignment. \textbf{Second}, we
establish TTSA as the principled Jacobian-free solver for RF-BO, deriving non-asymptotic
convergence rates of $\mathcal{O}(T^{-a})$ for $a \in (1/2,1)$ under strong convexity and
$\mathcal{O}(T^{-(1-a)})$ under the Polyak-\L{}ojasiewicz condition, both under general
Markovian noise, proving it a robust alternative to implicit differentiation.
\textbf{Third}, we validate across tasks from ODE control to WGAN training, achieving an
11.1\% reduction in Wasserstein distance for GANs and a 21.9\% improvement in entropy
stability in RL over standard baselines.

\begin{table}[htbp]
\caption{Comparison of stochastic algorithms for root-finding / equilibrium-seeking bilevel
optimization (RF-BO) in nonconvex settings, covering entropy tuning in RL, penalty
adaptation in GANs, ODE steady-state control, and KL penalty tuning in alignment.
Implicit Gradient Methods: squared-residual minimization with (approximate) implicit
differentiation~\citep{kwon2023penalty,hu2023contextual,giovannelli2025inexact};
Penalty-based Methods: penalty formulations with implicit
gradients~\citep{kwon2023penalty}; Contextual Methods: contextual bilevel
approaches~\citep{hu2023contextual}; Finding Small Hypergradients: methods targeting
small or vanishing hypergradients~\citep{chen2024finding}. $\tilde{O}$ hides polylog
factors in $1/\epsilon$; SC = strongly convex; PL = Polyak-\L{}ojasiewicz; Heavy-tailed
= bounded $p$-th moment with $p \in (1,2]$ (possibly infinite variance). RF-TTSA is the
only method that structurally escapes the Variance Trap via direct residual updates
without Jacobian estimation.}

\vspace{1.5ex}
\label{tab:rfbo_complexity_comparison}
\centering
\small
\resizebox{\columnwidth}{!}{
\begin{tabular}{lccc}
\noalign{\hrule height 1.2pt}
\rowcolor{gray!15}
Method & Sample Complexity & (UL) $h$ / Objective & (LL) $R$ / $g$ \\
\midrule
Implicit Gradient Methods & $\tilde{O}(\epsilon^{-3})$ -- $\tilde{O}(\epsilon^{-4})$ & Minimization of $\|h\|^2$ & SC or PL \\
Penalty-based Methods     & $\tilde{O}(\epsilon^{-3})$ & Minimization + penalty & SC \\
Contextual Methods        & $\tilde{O}(\epsilon^{-3})$ & Minimization (contextual) & SC \\
Finding Small Hypergradients & problem-dependent & Minimization of residual-like & SC or PL \\
\rowcolor{blue!10}
RF-TTSA (Ours) & $O(\epsilon^{-2})$ \text{ or } $\tilde{O}(\epsilon^{-\frac{1}{1-a}})$ & Direct root-finding $h=0$ & SC or PL \\
\midrule
\rowcolor{gray!15}
Method & Noise Assumption & Jacobian needed? & Single-Loop \\
\midrule
Implicit Gradient Methods & Bounded variance & Yes & Partial \\
Penalty-based Methods     & Bounded / small variance & Yes & Yes \\
Contextual Methods        & Small variance $O(\epsilon)$ & Yes & Yes \\
Finding Small Hypergradients & Bounded variance & Yes & Partial \\
\rowcolor{blue!10}
RF-TTSA (Ours)            & Bounded variance / Markovian & \textbf{No} & Yes \\
\noalign{\hrule height 1.2pt}
\end{tabular}
}
\end{table}

\section{Related Work}
Bilevel optimization employs two hypergradient paradigms. Implicit Differentiation (AID)
uses the Implicit Function Theorem to form Hessian-based linear systems
\citep{pedregosa2016hyperparameter, grazzi2020iteration}, offering strong theoretical
guarantees at high computational cost \citep{lorraine2020optimizing}; warm starts and
amortization help but overhead remains significant \citep{arbel2021amortized,
bertrand2020implicit}. Iterative Differentiation (ITD) unrolls optimization dynamics to
bypass matrix inversion \citep{franceschi2017forward}, improving scalability but
introducing instability in stochastic, non-convex settings \citep{ji2021bilevel,
yang2021provably}. Recent advances address stochastic constraints via penalty methods
\citep{kwon2023penalty, giovannelli2025inexact} or contextual formulations
\citep{hu2023contextual}, yet all rely on estimating hypergradients, which are
computationally hard and numerically unstable to minimize \citep{chen2024finding}.

To avoid nested loops, single-loop and variance-reduced methods reformulate the updates:
SOBA \citep{dagreou2022framework} and FSLA \citep{li2022fully} co-evolve variables to
achieve $O(1/\sqrt{T})$ convergence under mild conditions \citep{ghadimi2018approximation,
arjevani2023lower}, though variance reduction via STORM \citep{yang2021provably,
liu2022stochastic} still relies on convexity or well-conditioned Hessians
\citep{chen2021tighter}.

Two-time-scale stochastic approximation (TTSA), with theoretical roots in actor-critic
algorithms \citep{konda1999actor}, provides the natural foundation for RF-BO:
\citet{dalal2018finite} provided finite-sample analyses in reinforcement learning;
\citet{kaledin2020finite} established finite-time bounds under Markovian noise;
\citet{deb2025multitimescalestochasticapproximation} extended stability to general
multi-timescale settings; and \citep{dalal2018finite, doan2022nonlinear, hu2024central}
confirm TTSA's efficacy under general noise models. Our work builds on these foundations
to address variance amplification in the RF-BO structure. TTSA inherently controls
residuals via its tailored step-size scheme and handles heavy-tailed noise, ensuring
asymptotic stability despite biased updates and non-Gaussian perturbations
\citep{gorbunov2020stochastic}.

The RF-BO structure is widespread: SAC temperature tuning \citep{wang2020meta}, RLHF KL
penalty adjustment \citep{ouyang2022training, ziegler2019fine}, adaptive Huber regression
\citep{sun2020adaptive}, constrained RL \citep{tessler2018reward}, fairness-aware learning
\citep{zafar2017fairness}, and robust optimization \citep{zhang2022revisiting}. Yet existing
SBO methods assume upper-level minimization and lack analyses tailored to root-finding
dynamics---a gap our TTSA framework addresses. Concurrently and independently,
\citet{ConcurrentWork2} propose a distribution-aware robust bilevel optimization framework
using quantile-guided Huber updates within TTSA to improve robustness against heavy-tailed
noise. Code is provided at the anonymous link (Appendix~\ref{app:code_link}).

\section{RF-BO: Formulation and Jacobian-Free Dynamics}

\subsection{Problem Formulation}

We formalize Root-Finding Bilevel Optimization (RF-BO), a class distinct from conventional
bilevel optimization (BO). Rather than minimizing an upper-level scalar loss, RF-BO enforces
a stochastic root-finding condition directly: with $\alpha \in \mathcal{A}\subset\mathbb{R}^d$
(upper) and $\theta\in\Theta\subset\mathbb{R}^p$ (lower),
\begin{equation}
\label{eq:rfbo}
\text{Find } \alpha^\star \in \mathcal{A}, \quad \text{such that } h(\alpha^\star,
\theta^\star(\alpha^\star)) = 0, \quad \theta^\star(\alpha) = \arg\min_{\theta \in \Theta}
R(\alpha, \theta).
\end{equation}
This eliminates the need to compute implicit gradients $\nabla_\alpha \theta^\star(\alpha)$.
Such structure arises naturally in robust regression, median-type conditions, reinforcement
learning, temperature tuning~\citep{dalal2018finite}, representation learning, and moment
matching; the squared-residual reformulation $\min_\alpha \lVert
h(\alpha,\theta^\star(\alpha))\rVert^2$ reintroduces gradients and suffers the Variance Trap
under non-i.i.d.\ sampling.

\subsection{Jacobian-Free Update Rules}

We adopt Two-Time-Scale Stochastic Approximation (TTSA)~\citep{doan2022nonlinear,
hu2024central}, which separates the system's dynamics via distinct step sizes. Given
stochastic estimates $\widehat{R}$ and $\widehat{h}$, TTSA updates:
\begin{align}
\theta_{t+1} &= \Pi_\Theta \Big[ \theta_t - \eta_t \nabla_\theta \widehat{R}(\alpha_t,
\theta_t) \Big], \label{eq:theta-update} \\
\alpha_{t+1} &= \Pi_\mathcal{A} \Big[ \alpha_t - \gamma_t \, \widehat{h}(\alpha_t, \theta_t)
\Big], \label{eq:alpha-update}
\end{align}
with timescale separation $\gamma_t/\eta_t\to 0$ as $t\to\infty$ (a condition on asymptotic
decay rates, independent of initial constants $c_\gamma, c_\eta > 0$), where $\Pi_\Theta$
and $\Pi_\mathcal{A}$ are Euclidean projections onto closed convex sets. Unlike
minimization-based BO requiring hypergradients (often involving inverse Hessians),
Eq.~\eqref{eq:alpha-update} updates $\alpha$ directly via the stochastic residual
$\widehat{h}$, treating the upper level as a fixed-point iteration.

Algorithm~\ref{alg:ttsa} summarizes the procedure: each iteration uses lower-level
gradients for fast $\theta$-updates and upper-level residuals for slow $\alpha$-updates.
A key concern is whether Jacobian-free updates can navigate non-monotone or rotational
vector fields (e.g., oscillating GAN behavior) where implicit differentiation typically
corrects the update direction. In stochastic regimes, the Variance Trap dominates geometric
correction: noise amplification from estimating the implicit Jacobian destabilizes LSE
methods, whereas RF-BO's structural simplicity ensures convergence
(Figure~\ref{fig:robustness_ellipse_main}; Appendix~\ref{app:robustness_ellipse}).

\begin{algorithm}[tb]
    \caption{Robust Jacobian-Free TTSA for RF-BO}
    \label{alg:ttsa}
    \begin{algorithmic}[1]
        \STATE \textbf{Input:} Initial $\alpha_0, \theta_0$; sets $\Theta, \mathcal{A}$; Operator $\mathcal{T}_t$(Clipping); 
        \STATE \textbf{Initialize:} Step sizes $\eta_t, \gamma_t$ with $\gamma_t = o(\eta_t)$ 
        \FOR{$t = 0, 1, 2, \dots, T-1$}
            \STATE Sample stochastic batch $\xi_t \sim \mathcal{D}$
            \STATE \textbf{Stage 1: Fast Adaptation (Lower Level)}
            \STATE Compute gradient estimator $g_{\theta, t} = \nabla_\theta \widehat{R}(\alpha_t, \theta_t; \xi_t)$
            \STATE $\theta_{t+1} \leftarrow \Pi_\Theta [ \theta_t - \eta_t g_{\theta, t} ]$
            \STATE \textbf{Stage 2: Robust Root Tracking (Upper Level)}
            \STATE Compute residual estimator $v_{\alpha, t} = \widehat{h}(\alpha_t, \theta_t; \xi_t)$ 
            \GrayComment{\textbf{NO} implicit Jacobian/Hessian required}
            \STATE $\alpha_{t+1} \leftarrow \Pi_\mathcal{A} [ \alpha_t - \gamma_t \mathcal{T}_t(v_{\alpha, t}) ]$
        \ENDFOR
    \end{algorithmic}
    \vspace{-0.1cm}
\end{algorithm}

\begin{wrapfigure}{r}{0.45\columnwidth}
    \vspace{-0.5cm}
    \centering
    \includegraphics[width=0.47\columnwidth]{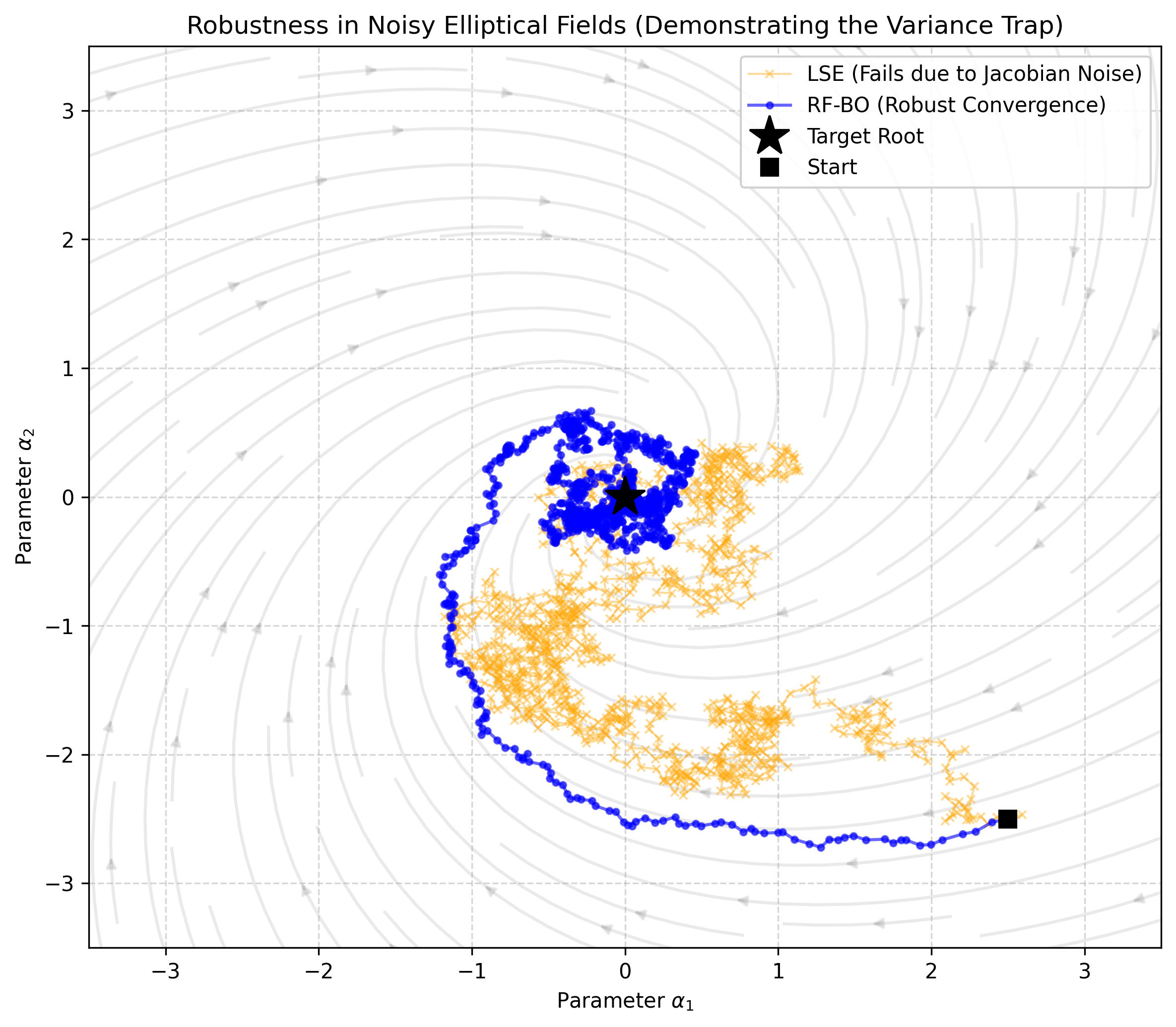}
    \caption{Escaping the Variance Trap in a noisy elliptical field.}
    \label{fig:robustness_ellipse_main}
    \vspace{-0.4cm}
\end{wrapfigure}

\paragraph{Robustness of Jacobian-Free Dynamics.}
As illustrated in Figure~\ref{fig:robustness_ellipse_main}, while LSE (Orange) attempts to use curvature to manage the rotational field, the noise in Jacobian estimation triggers variance explosion, leading to divergence. In contrast, RF-BO (Blue) relies on Jacobian-free first-order dynamics, successfully filtering noise and spiraling steadily towards the equilibrium. This paradigm is strengthened by 
finite-sample guarantees~\citep{dalal2018finite}, variance reduction~\citep{lan2020first}, 
and CLTs under Markovian noise~\citep{hu2024central}, indicating TTSA can attain $O(1/T)$ 
rates and outperform single-timescale or residual-based methods. Our convergence analysis 
assumes Lipschitz continuity for $h$ and $R$, bounded variance, compact domains 
$\mathcal{A}, \Theta$, Markovian noise~\citep{dalal2018finite,hu2024central}, and 
Robbins-Monro step sizes $\sum_t \eta_t=\infty$, $\sum_t \eta_t^2<\infty$ with 
$\gamma_t/\eta_t\to 0$ for equilibrium tracking~\citep{konda1999actor}.

\subsection{Analysis: Escaping Noise-Geometry Traps}
\label{subsec:mechanism}

While standard clipping stagnates where gradients vanish relative to heavy-tailed
noise~\cite{cutkosky2023optimal,nguyen2023improved}, RF-BO's residual-driven mechanism
normalizes updates via dynamic quantiles $\psi_k$ to tunnel through variance-induced
barriers (Figure~\ref{fig:energy_barrier}), actively utilizing the persistent residual
vector $h(\alpha)$ to escape sub-optimal basins beyond mere outlier suppression
(Appendix~\ref{app:potentialwell}).

\begin{figure}[ht]
    \centering
    \includegraphics[width=0.80\linewidth]{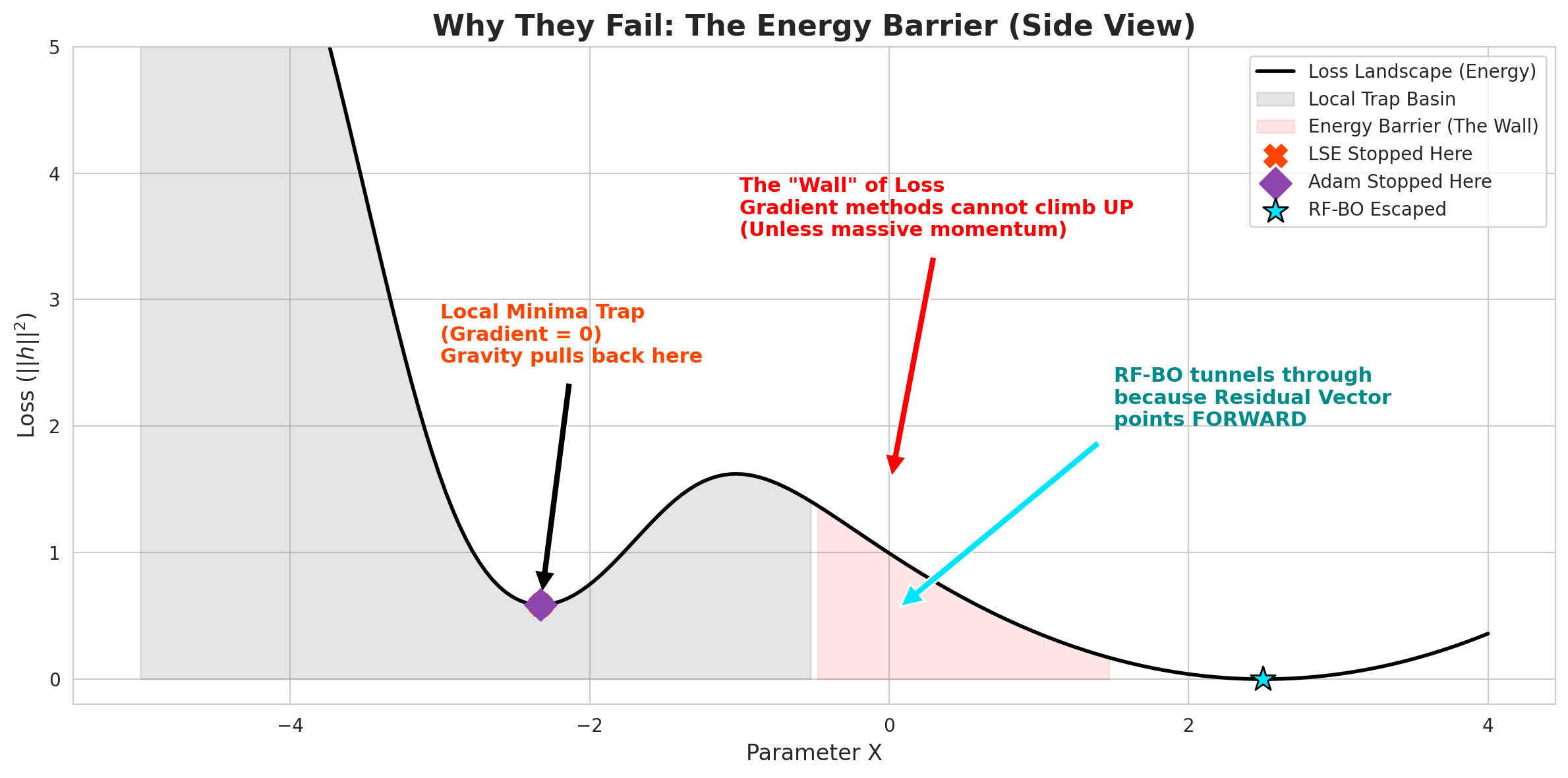} 
    \vspace{-2mm}
    \caption{\textbf{Mechanism of Escape.} While gradient-based methods (LSE/Adam) are constrained by the energy landscape and trapped in local basins (left), RF-BO is driven by the residual vector field. This allows it to ignore the energy barrier (The Wall) and traverse towards the global root, demonstrating robustness against both heavy-tailed outliers and deceptive local geometry.}
    \label{fig:energy_barrier}
    \vspace{-4mm}
\end{figure}

\section{Unified Applications of the RF-BO Framework}
RF-BO unifies adaptive tasks by substituting minimization with Jacobian-free equilibrium
seeking. \textbf{Maximum Entropy RL:} In Soft Actor-Critic~\citep{haarnoja2018soft},
temperature $\alpha$ is tuned via
\[h(\alpha,\theta)=\mathbb{E}_{s,a}[-\log\pi_\theta(a\mid s)]-\mathcal{H}_{\text{target}}=0;\]
when $h < 0$, the update $\alpha \leftarrow \alpha - \gamma_t h$ increases $\alpha$,
promoting more exploration. While the SAC heuristic adopts a similar direct update, RF-BO
provides the theoretical basis explaining why this structurally avoids variance amplification
in squared-residual reformulations (Proposition~\ref{prop:variance}). \textbf{Robust
Contrastive Learning:} For SimCLR~\citep{wang2020understanding}, RF-BO frames optimal
temperature $\tau$ as satisfying the KKT condition $h(\tau,\theta)=0$, enabling stable
auto-tuning without second-order derivatives.

\textbf{Stabilizing GAN Training:} In WGAN-GP~\citep{gulrajani2017improved}, penalty
$\lambda$ is tuned via
\[h(\lambda,\theta)=1-\mathbb{E}_{\hat{x}}\big[\|\nabla_{\hat{x}}D_\theta(\hat{x})\|_2\big]=0,\]
avoiding mode collapse (Sec.~\ref{sec:experiments}) and acting as a regularization
controller in non-convex landscapes (Sec.~\ref{sec:theory}). \textbf{Adaptive KL Penalty
in LLM Alignment:} For RLHF~\citep{ouyang2022training}, RF-BO tunes KL penalty $\beta$ via
\[h(\beta,\theta)=\mathrm{KL}_{\text{target}}-\mathbb{E}_{s\sim
d^{\pi_\theta}}\big[\mathrm{KL}(\pi_\theta\|\pi_{\text{ref}})\big]=0,\]
bypassing the computationally prohibitive implicit hypergradients required for large-scale
alignment.

\section{Theoretical Analysis}
\label{sec:theory}

We present the theoretical foundations of our TTSA framework for RF-BO, establishing three
core results: (1) a Variance Comparison (Proposition~\ref{prop:variance}) showing TTSA's
advantage over squared-residual minimization; (2) a Non-Asymptotic Convergence Rate under
strong convexity (Theorem~\ref{thm:strong_convex_convergence}); and (3) a Convergence Rate
under the Polyak-\L{}ojasiewicz (PL) Condition (Theorem~\ref{thm:pl_convergence}), relaxing
convexity for deep learning applications \citep{ji2021bilevel, hong2023two}.

\subsection{Assumptions}

Core assumptions, standard in the stochastic approximation and bilevel optimization
literature (e.g., \citealp{ghadimi2018approximation, ji2021bilevel, kaledin2020finite}):

\begin{assumption}[Smoothness and Geometry]
\label{as:smooth_convex}
\begin{enumerate}
    \item \textit{Regularity:} $\nabla_\theta R(\theta, \alpha)$ and $h(\alpha, \theta)$ are
    $L$-Lipschitz continuous w.r.t.\ $(\theta, \alpha)$.
    \item \textit{Strong Convexity or PL Condition:} For any fixed $\alpha \in \mathcal{A}$,
    $R(\cdot, \alpha)$ is either $\mu$-strongly convex
    (Theorem~\ref{thm:strong_convex_convergence}) or satisfies the $\mu$-PL condition
    $\frac{1}{2}\|\nabla_\theta R(\theta, \alpha)\|^2 \ge \mu (R(\theta, \alpha) -
    R^*(\alpha))$, where $R^*(\alpha) = \min_{\theta} R(\theta, \alpha)$.
\end{enumerate}
\end{assumption}

\begin{assumption}[Stochastic Noise]
\label{as:noise}
The stochastic estimates $G_{\theta,t}$ and $G_{\alpha,t}$ are unbiased w.r.t.\ the
filtration $\mathcal{F}_t$, with uniformly bounded conditional second moments. For the
variance analysis, we assume independent noise; for convergence, we allow Markovian noise
common in RL applications \citep{dalal2018finite, hu2024central}.
\end{assumption}

\begin{assumption}[Limit ODE Stability]
\label{as:ode}
The ODE governing the slow timescale, $\frac{d\alpha}{dt} = -h(\alpha, \theta^*(\alpha))$,
admits a unique and globally asymptotically stable equilibrium $\alpha^*$, providing the
necessary Lyapunov stability. This is a standard condition for TTSA
\citep{konda1999actor, doan2022nonlinear}.
\end{assumption}

\subsection{Variance Comparison}

We first quantify the Variance Trap, highlighting the structural advantage of RF-BO over
naive minimization.

\begin{proposition}[Variance Comparison]
\label{prop:variance}
Let $L_\alpha(\alpha) := \frac{1}{2}\|h(\alpha, \theta^*(\alpha))\|^2$. Under
Assumption~\ref{as:noise}, let $H = \|h\|$ and $J = \|\nabla_\alpha h\|$.
\begin{enumerate}
    \item The conditional variance of the TTSA update is bounded:
    $\mathbb{V}[G_{\alpha,t} \mid \mathcal{F}_t] \le \sigma_\alpha^2$.
    \item The variance of the squared-residual gradient $\widehat{\nabla}L_\alpha$ satisfies:
    $$ \mathbb{V}[\widehat{\nabla}L_\alpha \mid \mathcal{F}_t] \le J^2\sigma_\alpha^2 +
    H^2\sigma_{\nabla h}^2 + \sigma_{\nabla h}^2 \sigma_\alpha^2. $$
\end{enumerate}
\end{proposition}
\textit{Implication:} The squared-residual variance scales with $H^2$, causing numerical
instability when the residual is large. TTSA maintains bounded variance independent of $H$,
providing a key remedy to the instabilities observed in our ODE and GAN experiments.

\subsection{Convergence Analysis}

We provide non-asymptotic rates for both strongly convex and PL settings. Using Generalized
Young's Inequality, these results hold for any finite Lipschitz constant $L$ without
restrictions \citep{doan2022nonlinear}.

\begin{theorem}[Convergence under Strong Convexity]
\label{thm:strong_convex_convergence}
Under Assumptions~\ref{as:smooth_convex}(a-Strong Convexity), \ref{as:noise},
and~\ref{as:ode}, let step sizes be $\eta_t = c_\eta (t+t_0)^{-a}$ and $\gamma_t = c_\gamma
(t+t_0)^{-1}$ with $a \in (1/2, 1)$ and $c_\eta \mu > 1$. Then:
$$ \mathbb{E}[\|\theta_T - \theta^*(\alpha_T)\|^2] + \mathbb{E}[\|\alpha_T - \alpha^*\|^2]
\le \mathcal{O}(T^{-a}). $$
\end{theorem}

This rate confirms TTSA's Hessian-free efficiency. To account for non-convex lower-level
problems, e.g., neural networks in our SimCLR/GAN experiments, we extend to the PL condition
\citep{karimi2016linear}.

\begin{theorem}[Convergence under PL Condition]
\label{thm:pl_convergence}
Under Assumptions~\ref{as:smooth_convex}(a-PL), \ref{as:noise}, and~\ref{as:ode}, with step
sizes $\eta_t = c_\eta (t+t_0)^{-a}$ ($a \in (1/2, 1)$) and $\gamma_t = \mathcal{O}(1/t)$,
the average gradient norm and upper-level error satisfy:
$$ \frac{\sum_{t=1}^{T} \eta_t \mathbb{E}[\|\nabla_\theta R(\theta_t, \alpha_t)\|^2]}
{\sum_{t=1}^{T} \eta_t} + \mathbb{E}[\|\alpha_T - \alpha^*\|^2] \le \mathcal{O}(T^{-(1-a)}). $$
By choosing $a \approx 1/2$, we recover the standard $\mathcal{O}(1/\sqrt{T})$ rate for the
weighted average stationarity measure in stochastic non-convex optimization
\citep{ji2021bilevel}, consistent with the empirical robustness observed in deep models.
\end{theorem}

\begin{theorem}[Robust Convergence under Heavy-Tailed Noise]
\label{thm:heavy_tail}
Suppose the noise assumption is relaxed such that the stochastic estimates $G_t$ only have a
bounded $(1+\delta)$-th moment for some $\delta \in (0, 1]$ (allowing for infinite
variance). If the updates are modified with a dynamic clipping operator $\tilde{G}_t := G_t
/ \max(1, \|G_t\|/B_t)$ where $B_t \to \infty$ sufficiently slowly, then under
Assumptions~\ref{as:smooth_convex} and~\ref{as:ode}, the upper-level iterate $\alpha_t$
converges almost surely to $\alpha^*$.
\end{theorem}

\textit{Remark:} This theorem theoretically justifies RF-BO's stability against impulsive
outliers (e.g., RL rewards), preventing the gradient explosion and divergence of
squared-residual minimization; see visual proof in App.~\ref{app:atmospheric_shielding}.

\section{Experiments}
\label{sec:experiments}

We evaluate the proposed TTSA framework across three distinct domains:
(1) Synthetic Analysis \& ODE Systems: validating variance amplification theory and proving
efficacy in ODE-driven control tasks against recent implicit-gradient baselines
\cite{kwon2023penalty,chen2024finding,hu2023contextual,giovannelli2025inexact};
(2) Reinforcement Learning \& Game Theoretic Equilibrium: demonstrating performance in
complex environments including SAC temperature tuning and Multi-Agent Nash Equilibrium
seeking;
(3) Generative \& Representation Learning: scalability in deep learning via WGAN-GP
stabilization and SimCLR auto-tuning. Appendix~\ref{app:robustness_ellipse} verifies RF-BO
in physical systems.

\subsection{Synthetic Analysis and ODE Systems}

\subsubsection{Variance Amplification in Linear RF-BO}
To empirically validate Proposition~\ref{prop:variance}, we construct a synthetic non-linear
RF-BO task with an ill-conditioned landscape. The lower-level objective is a regularized
regression $R(\theta, \alpha) := \mathbb{E}[ \frac{1}{2} (x^\top \theta - \alpha)^2 ] +
\frac{\lambda}{2} \|\theta\|_2^2 + \frac{\kappa}{4} \|\theta\|_4^4$, where $\alpha$ is a
learnable target, and the upper-level enforces the linear moment condition:
\[
h(\alpha, \theta) = \mathbb{E}[\theta_{\text{true}}^\top \theta^*(\alpha)] - C = 0.
\]
We benchmark RF-BO against squared-residual minimization (Opt-h2) and a Single-Scale
baseline, correlating update variance with residual magnitude to directly test the Variance
Trap hypothesis (Appendix~\ref{app:synthetic_details}).

Table~\ref{tab:results} shows RF-BO achieves the lowest final error ($0.2301 \pm 0.01$),
markedly outperforming Opt-h2 and Single-Scale baselines. Figure~\ref{fig:variance_residual}
(Top) confirms RF-BO mitigates variance amplification with update variances orders of
magnitude lower than Opt-h2 ($1.46 \times 10^{-2}$ vs $1.62 \times 10^{2}$);
Figure~\ref{fig:variance_residual} (Bottom) attributes this to strict alignment with the
theoretical ideal (slope $1.00$), contrasting Opt-h2's unstable negative scaling ($-1.88$).

\subsubsection{Stochastic Steady-State Control}
To compare RF-BO with recent implicit-gradient-based methods
\citep{kwon2023penalty, chen2024finding, hu2023contextual, giovannelli2025inexact},
we implement a non-linear stochastic steady-state control task. The system evolves
according to an It\^{o} stochastic differential equation with tanh dynamics:
\begin{equation}
\frac{dx}{dt} = -\alpha \tanh(x(t)) + u(t), \quad u(t) \sim \mathcal{N}(1, \sigma^2).
\label{eq:ode_sde}
\end{equation}
This maps to the RF-BO formulation (Eq.~\eqref{eq:rfbo}) as follows. The lower-level variable
$\theta \equiv x_{ss}(\alpha)$ denotes the steady-state value of $x$, identified as
the minimizer of $R(\alpha, \theta) = \mathbb{E}[(x(t) - x_{ss}(\alpha))^2]$, so that
$\theta^*(\alpha) = x_{ss}(\alpha)$. The upper-level root-finding condition enforces
the steady-state target:
\[
h(\alpha, \theta) = \mathbb{E}[x_{ss}(\alpha)] - x_{\mathrm{target}} = 0.
\]
The tanh non-linearity introduces saturation regions where gradients vanish, severely
challenging minimization baselines. Table~\ref{tab:ode_control} reveals the decisive advantage of the root-finding formulation.
RF-BO converges rapidly (88 episodes) to a perfect root ($|h| \approx 0.001$). In contrast,
LSE gets stuck at $|h| \approx 2.000$, failing completely, confirming that minimizing
squared residuals ($h^2$) suffers from vanishing gradients in the tanh saturation region.
Implicit gradient methods \cite{giovannelli2025inexact} perform better than LSE but still
exhibit higher variance or slower convergence due to Hessian estimation difficulty in
non-linear systems.

\begin{wrapfigure}{r}{0.45\columnwidth}
    \vspace{-0.4cm}
    \centering
    \includegraphics[width=0.45\columnwidth]{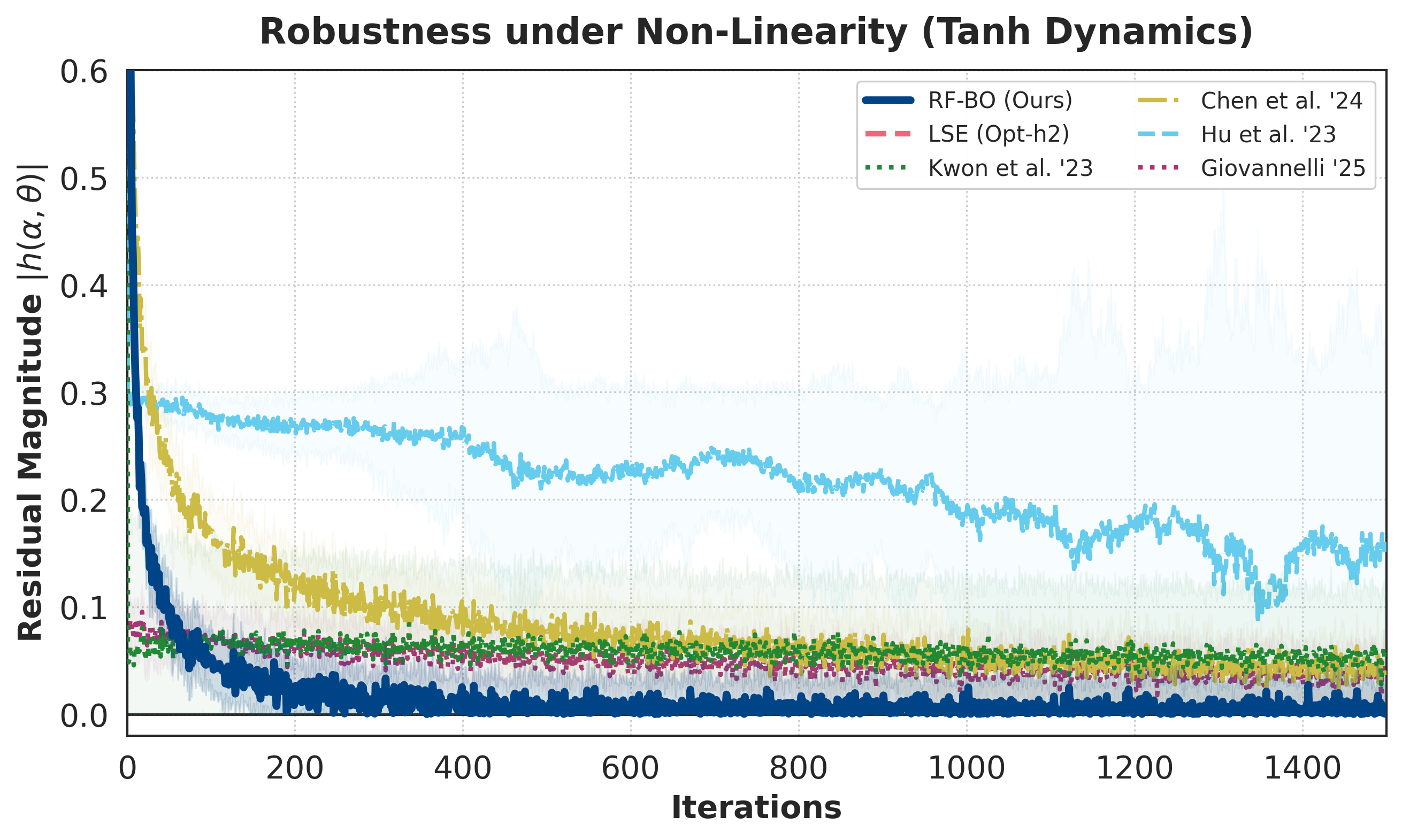}
    \caption{Convergence trajectories on the ODE control task under non-linear tanh dynamics.}
    \label{fig:ode_trajectory}
    \vspace{-0.4cm}
\end{wrapfigure}

\paragraph{Trajectory Analysis under Saturation.}
Figure~\ref{fig:ode_trajectory} visually confirms these results. While \citet{hu2023contextual}
and \citet{kwon2023penalty} improve upon naive LSE, they still suffer from high-frequency
noise from implicit differentiation under local non-linearities. TTSA's direct root-finding
update bypasses this, yielding a stable (lowest variance) and accurate (zero bias) trajectory.
While RF-BO demonstrates rapid and stable convergence to the zero residual, the
implicit gradient baselines \citep{kwon2023penalty, chen2024finding, hu2023contextual,
giovannelli2025inexact} exhibit significant oscillation and slower rates due to hypergradient
estimation noise. In contrast, LSE (Orange) stabilizes with a persistent non-zero bias,
highlighting the failure of residual minimization in saturation regions.

\subsection{Reinforcement Learning and Equilibrium}

\subsubsection{SAC Temperature Tuning}
We apply RF-BO to automatic temperature tuning in Soft Actor-Critic (SAC) on Pendulum-v1,
comparing against a fixed temperature baseline (Fixed-Temp) and the standard single-timescale
adaptive method (Original-SAC) \citep{haarnoja2018soft}. Pendulum-v1 isolates the variance
reduction mechanism of the root-finding formulation in a controlled setting; scalability to
higher-dimensional environments is validated in Section~\ref{sec:mujoco}. Unlike standard
approaches that treat temperature adjustment as gradient-based minimization, RF-BO formulates
it as a root-finding problem on the entropy constraint, solved via Jacobian-free TTSA.

Table~\ref{tab:sac_summary_vertical} and Figure~\ref{fig:EXP2_1} compare three temperature
tuning strategies. Fixed-Temp excels early but lacks adaptability (entropy deviation:
$0.716$); Original-SAC improves control yet exhibits slower convergence, with both baselines
plateauing at suboptimal returns ($-279$ and $-278$). Fixed-Temp initially leads via
intensive exploration ($-936$ at 5k steps), but RF-BO overtakes baselines by 20k steps
($-276$ vs.\ $-296$ for Orig-SAC) and expands the gap by 30k steps, reaching $-200$
relative to the gradient-based $-256$, with a comparable final return of $-251$. The core
advantage lies in stability: while both adaptive methods converge to a similar temperature
scale ($\alpha_{\text{final}} \approx 0.17$), RF-BO achieves a markedly lower entropy
deviation ($0.174$ vs.\ $0.223$), indicating that root-finding mitigates the oscillations
inherent in gradient-based dual descent.

\subsubsection{Scalability Validation: MuJoCo Continuous Control}
\label{sec:mujoco}
To validate RF-BO beyond low-dimensional settings, we conduct experiments on MuJoCo
continuous control (HalfCheetah-v4), comparing against two implicit-gradient baselines
\citep{kwon2023penalty, giovannelli2025inexact} over 5 random seeds, evaluated at 200,000
environment steps. The upper-level $\alpha$ controls the entropy temperature via the same
root-finding constraint as Section~\ref{sec:experiments}, updated via Jacobian-free TTSA.

The results confirm RF-BO's structural advantage in high-variance settings.
\citet{kwon2023penalty} achieves peak returns on seeds 42 and 44 but collapses
catastrophically on seeds 43 and 45 (returns of $31.2$ and $12.7$), reflecting
initialization sensitivity from implicit Jacobian estimation under high-dimensional
stochastic gradients. \citet{giovannelli2025inexact} fails to learn on three of five seeds.
RF-BO prevents divergence across all seeds, demonstrating the robustness of Jacobian-free
updates in complex, high-dimensional environments.

\subsubsection{Multi-Agent Equilibrium}
To test RF-BO in a complex equilibrium setting, we employ a Multi-Agent RL task. In a
zero-sum GridWorld game (5x5), the upper-level $\alpha$ enforces the value equilibrium
$h(\alpha, \theta) = V_1(\theta_1^*) - V_2(\theta_2^*) = 0$, directly simulating the ODE
$\dot{\alpha}(t) = -h(\alpha, \theta^*(t))$. We benchmark over 10 seeds against LSE and the
four referenced methods.

Table~\ref{tab:marl_benchmark} highlights three key findings:
(1) Stability: RF-BO achieves the lowest variance ($8.6 \times 10^{-4}$), surpassing the
most competitive implicit baselines, e.g.\ \cite{hu2023contextual} by $\sim$16\%, minimizing
noise inherent in multi-agent dynamics.
(2) Accuracy vs.\ Bias: while LSE and \cite{chen2024finding} converge to biased solutions
($>0.04$) due to vanishing gradients, RF-BO matches the precision of Hessian-based methods
($|h| \approx 0.006$) without their computational overhead.
(3) Efficiency: RF-BO attains this via simple Jacobian-free updates, matching the optimal
convergence speed (60 eps) of implicit solvers.

\subsection{Generative and Representation Learning}

\subsubsection{Stabilizing WGAN-GP}
We apply RF-BO to gradient penalty (GP) regularization in WGANs, where $\lambda$ must
satisfy the Lipschitz constraint $h(\lambda, \theta) = \mathbb{E}[\|\nabla_{\hat{x}}
D(\hat{x})\|_2] - 1 = 0$. We benchmark against two baselines:
(1) LSE (True Grad): minimizing the squared violation $h^2$ via full automatic
differentiation;
(2) Dual Adam: an adaptive primal-dual method updating $\lambda$ via Adam to enforce the
constraint.

\begin{figure*}[htbp]
   \centering
    \includegraphics[width=\textwidth]{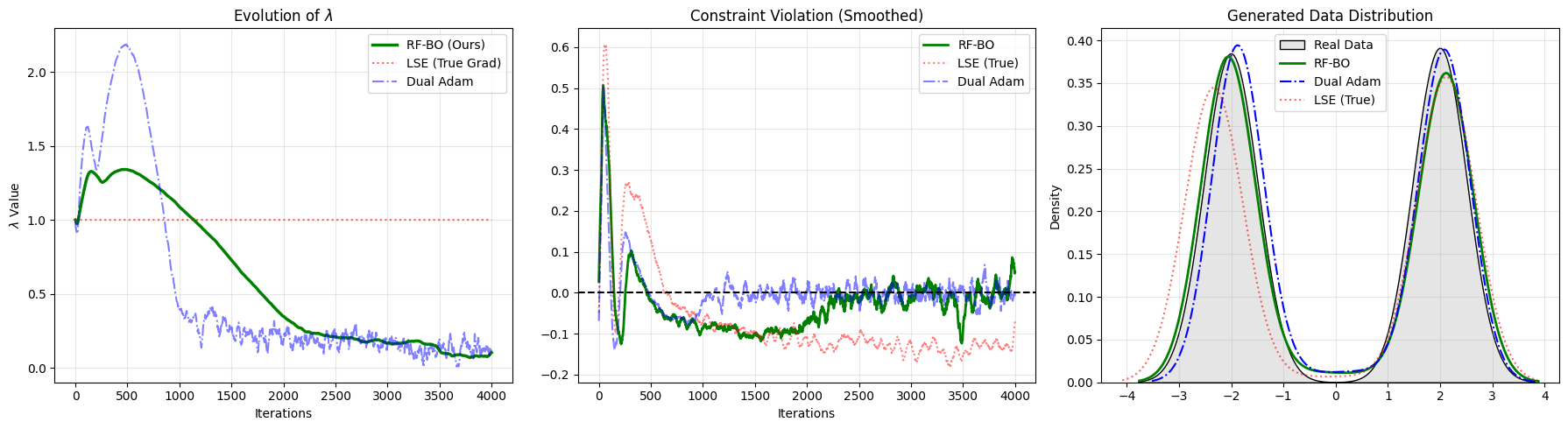}
\caption{WGAN-GP Stabilization Dynamics.
\textbf{(Left) $\lambda$ Evolution:} RF-BO (green) exhibits smooth, monotonic convergence,
whereas Dual Adam (blue) suffers significant oscillation from momentum instability, and
LSE (red) stagnates.
\textbf{(Middle) Constraints:} RF-BO consistently maintains the lowest constraint violation
($|h| \approx 0.097$), outperforming Dual Adam ($0.122$) and LSE ($0.131$).
\textbf{(Right) Quality:} The distribution from RF-BO (green) overlaps almost perfectly
with real data (black), achieving the lowest Wasserstein Distance (0.136). In contrast,
LSE (red/orange) shows severe mode shift, and Dual Adam (blue) exhibits slight misalignment.}
    \label{fig:wgan_results}
    \vspace{-0.3cm}
\end{figure*}

As shown in Figure~\ref{fig:wgan_results}, RF-BO achieves optimal distributional alignment
(WD = 0.136), outperforming Dual Adam (0.153) and significantly surpassing LSE (0.250),
which suffers from mode collapse. RF-BO enforces the tightest constraint satisfaction
(Mean $|h| \approx 0.097$ vs.\ Adam's $0.122$) with high stability: while Dual Adam
exhibits high-frequency oscillations from momentum overshooting, RF-BO operates with a
variance $\sim$500$\times$ lower ($9.9\text{e-}8$ vs.\ $5.6\text{e-}5$), confirming that
Jacobian-free updates eliminate the instability inherent in primal-dual methods.

\subsubsection{Contrastive Temperature Tuning}
We benchmark RF-BO for SimCLR temperature tuning against static Fixed-Temp,
gradient-based Original-Adaptive, heuristic Cosine-Decay, and Projected-SGD baselines.
RF-BO bypasses hypergradient estimation via a Jacobian-free, two-timescale update
satisfying the KKT root condition directly, decoupling temperature updates from volatile
local loss gradients.

As shown in Table~\ref{tab:main_results_detailed}, RF-BO yields a peak Top-1 accuracy of
$74.72\%$ with suppressed variance of $\pm 0.18\%$ (nearly five-fold lower than
Projected-SGD). Cosine-Decay's over-optimization of uniformity ($-2.79$) at $\tau=0.1$
compromises semantic alignment for a suboptimal $71.77\%$ accuracy, while Projected-SGD
saturates at $\tau=1.0$ and fails to capture fine-grained structures that RF-BO preserves
at the optimal discriminative equilibrium $\tau \approx 0.622$.

RF-BO strikes a superior trade-off between alignment ($0.025$) and uniformity ($-2.59$)
on the NT-Xent loss surface, maintaining semantic consistency alongside feature separation.
Its per-epoch runtime of $15.9$s is competitive with the $14.7$s heuristic baseline,
establishing high-precision convergence, structural variance reduction, and low computational
cost as a verifiable alternative for tuning hyperparameters in deep representation learning.

\section{Conclusion and Future Work}
\label{sec:conclusion}

In this paper, we formalize \textbf{Root-Finding Bilevel Optimization (RF-BO)} to tackle the \textbf{Variance Trap}, the instability from noise-amplified implicit Jacobians. We propose a Jacobian-free \textbf{Two-Time-Scale Stochastic Approximation (TTSA)} solver decoupling update variance from residual magnitude by updating hyperparameters along root errors instead of minimizing squared residuals.

Theoretically, we provided the first non-asymptotic analysis under general Markovian noise, 
establishing a convergence rate of $\mathcal{O}(T^{-a})$ for $a \in (1/2, 1)$ under the 
Strongly Convex condition and $\mathcal{O}(T^{-(1-a)})$ under the Polyak-\L{}ojasiewicz (PL) 
condition, thereby confirming that Jacobian-free updates maintain bounded variance even under 
large residuals, providing a rigorous alternative to the instability present in implicit differentiation.

Empirically, RF-BO demonstrates superior stability across chaotic non-linear ODE control, reinforcement learning, and generative modeling, achieving perfect convergence where baselines fail, a significant reduction in entropy 
deviation in SAC (from $0.223$ to $0.174$, a $21.9\%$ improvement in entropy stability), 
and an $11.1\%$ quality improvement in WGAN-GP. Diagnostics confirm these gains stem from a massive reduction in update variance achieved with negligible computational overhead ($O(C_G)$ vs $O(K \times C_G)$), validating its high-dimensional equilibrium seeking practicality; future research scales RF-BO to LLMs via TTSA and PEFT (e.g., LoRA) for KL-constrained alignment in over-parameterized regimes, bypassing the instability of implicit differentiation. Note that in deterministic or low-noise regimes, implicit gradients remain preferable as they provide curvature information that accelerates convergence without triggering the Variance Trap.

\bibliography{references}
\bibliographystyle{plainnat}

\newpage
\appendix

\section{Supplementary Experimental Figures and Tables}
\label{app:supplementary_figures_tables}

This section contains the detailed quantitative tables and supplementary convergence trajectories that complement the main text experimental analysis in Section \ref{sec:experiments}. 

\begin{figure}[htbp]
    \centering
    \includegraphics[width=0.7\textwidth]{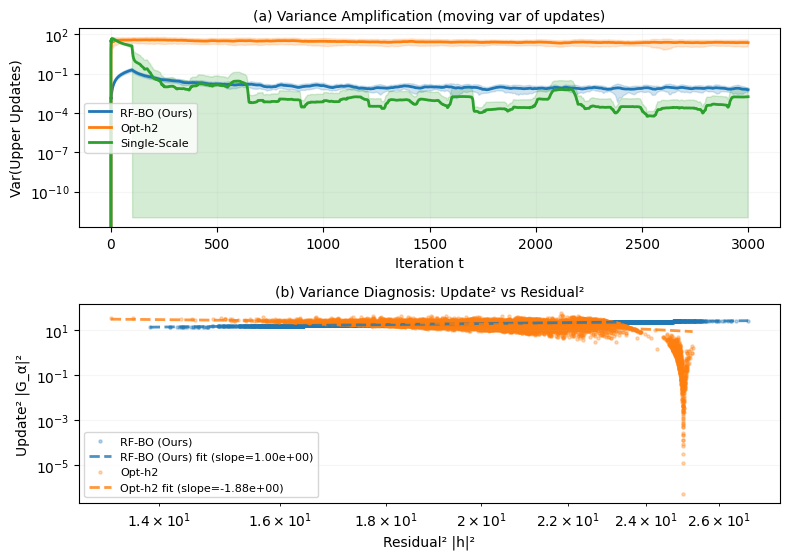} 
    \caption{Variance analysis (15 seeds) for Synthetic RF-BO. \textbf{(Top)} Moving update variance; RF-BO maintains stable, low variance. \textbf{(Bottom)} Update variance vs. squared residual. RF-BO exhibits ideal scaling (slope $\approx 1.0$), whereas Opt-h2 shows unstable behavior.}
    \label{fig:variance_residual}
\end{figure}

\begin{table}[htbp]
    \caption{\textbf{Synthetic RF-BO results (15 seeds).} RF-BO achieves optimal variance scaling (Slope 1.00) and significantly lower update variance, verifying the escape from the Variance Trap.}
    \label{tab:results}
    \begin{center}
    \begin{small}
    \renewcommand{\arraystretch}{1.25} 
    \setlength{\tabcolsep}{6pt}
    \resizebox{0.7\textwidth}{!}{
    \begin{tabular}{lcccc}
        \noalign{\hrule height 1.2pt}
        \rowcolor{gray!15} 
        Method 
        & $|\alpha_{\text{final}} - \alpha^*|$ $\downarrow$
        & Var$_{\text{upper}}$ $\downarrow$
        & Var$_{\text{ratio}}$ 
        & Slope \\
        \midrule
        Single-Scale 
        & 3.748 {\tiny ($\pm$ 0.11)} 
        & 5.16e-3 
        & 1.00 
        & N/A \\
        Opt-h2 (LSE)
        & 0.295 {\tiny ($\pm$ 0.03)} 
        & 1.62e+2 
        & 1.15 
        & -1.88 \\
        \rowcolor{blue!10} 
        RF-BO (Ours) 
        & \textbf{0.230} {\tiny ($\pm$ 0.01)} 
        & \textbf{1.46e-2} 
        & 1.00 
        & \textbf{1.00} \\
        \noalign{\hrule height 1.2pt}
    \end{tabular}
    }
    \end{small}
    \end{center}
\end{table}

\begin{table}[htbp]
    \caption{\textbf{ODE-Driven Control under Non-Linearity.} RF-BO achieves perfect convergence ($|h| \approx 0$) and the fastest speed. LSE fails to converge due to gradient bias. Conv. Eps denotes episodes required to reach a stable $\epsilon$-convergence.}
    \label{tab:ode_control}
    \begin{center}
    \begin{small}
    \renewcommand{\arraystretch}{1.25}
    \setlength{\tabcolsep}{4pt}
    \resizebox{0.7\textwidth}{!}{
    \begin{tabular}{lccc}
        \noalign{\hrule height 1.2pt}
        \rowcolor{gray!15} 
        Method 
        & Final $|h|$ $\downarrow$
        & Var($h$) $\downarrow$
        & Conv. Eps $\downarrow$ \\
        \midrule
        LSE (Opt-h2) 
        & 2.000 {\tiny ($\pm$ 0.000)} 
        & 0.0 
        & 2000 {\tiny ($\pm$ 0)} \\
        \cite{kwon2023penalty}
        & 0.049 {\tiny ($\pm$ 0.062)} 
        & 3.8e-3 
        & 754 {\tiny ($\pm$ 795)} \\
        \cite{hu2023contextual} 
        & 0.175 {\tiny ($\pm$ 0.165)} 
        & 2.7e-2 
        & 1537 {\tiny ($\pm$ 604)} \\
        \cite{chen2024finding} 
        & 0.037 {\tiny ($\pm$ 0.029)} 
        & 8.7e-4 
        & 645 {\tiny ($\pm$ 82)} \\
        \cite{giovannelli2025inexact}
        & 0.032 {\tiny ($\pm$ 0.027)} 
        & 7.1e-4 
        & 249 {\tiny ($\pm$ 156)} \\
        \rowcolor{blue!10} 
        RF-BO (Ours) 
        & \textbf{0.001} {\tiny ($\pm$ 0.028)} 
        & \textbf{7.6e-4} 
        & \textbf{88} {\tiny ($\pm$ 11)} \\
        \noalign{\hrule height 1.2pt}
    \end{tabular}
    }
    \end{small}
    \end{center}
\end{table}

\begin{figure}[htbp]
    \centering
    \includegraphics[width=0.99\textwidth]{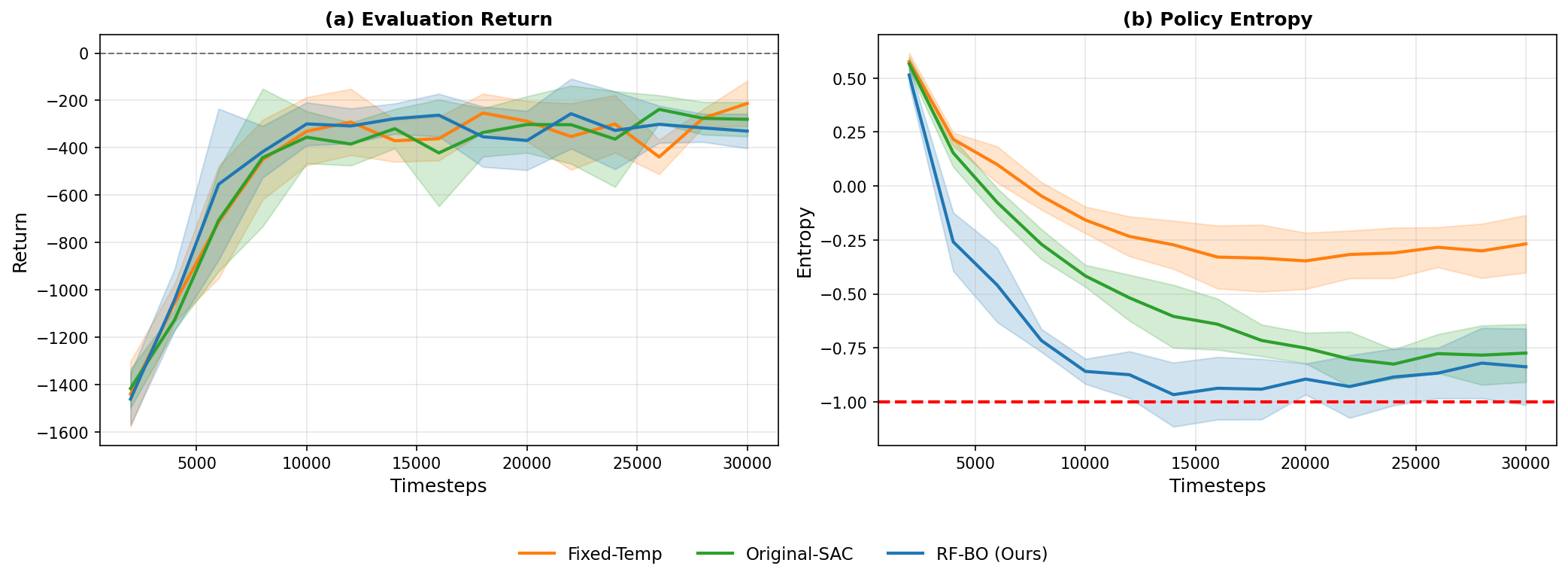} 
    \caption{\textbf{SAC temperature tuning on Pendulum-v1} (5 seeds). \textbf{(a)} Evaluation Return: RF-BO matches the performance of standard baselines with high stability. \textbf{(b)} Policy Entropy: RF-BO converges to the target (red dashed line) with the highest precision, validating the root-finding formulation.}
    \label{fig:EXP2_1}
\end{figure}

\begin{table}[htbp]
    \caption{\textbf{SAC Temperature Tuning Summary for Pendulum-v1.} All return values are averaged over multiple seeds (mean ± std) and rounded to the nearest integer, with best results highlighted in \textbf{bold}. Final return denotes the average of the last five evaluations.}
    \label{tab:sac_summary_vertical}
    \begin{center}
    \begin{small}
    \resizebox{0.5\textwidth}{!}{
    \begin{tabular}{lcc>{\columncolor{blue!10}}c}
        \noalign{\hrule height 1.2pt}
        \rowcolor{gray!15} 
        Metric & Fixed-T & Orig-SAC & RF-BO \\
        \midrule
        \multicolumn{3}{l}{\textit{Performance (Return)}} & \\
        \quad @ 5k Steps    & \textbf{-936{\tiny ($\pm$144)}} & -951{\tiny ($\pm$179)} & -996{\tiny ($\pm$118)} \\
        \quad @ 10k Steps   & -408{\tiny ($\pm$163)} & \textbf{-301{\tiny ($\pm$99)}}  & -320{\tiny ($\pm$130)} \\
        \quad @ 20k Steps   & -440{\tiny ($\pm$109)} & -296{\tiny ($\pm$102)}    & \textbf{-276{\tiny ($\pm$98)}} \\
        \quad @ 30k Steps   & -237{\tiny ($\pm$81)} & -256{\tiny ($\pm$108)}    & \textbf{-200{\tiny ($\pm$64)}} \\
        \quad Final         & -279{\tiny ($\pm$48)}  & -278{\tiny ($\pm$45)}    & \textbf{-251{\tiny ($\pm$38)}}  \\
        \midrule
        \multicolumn{3}{l}{\textit{Tuning \& Stability}} & \\
        \quad \(\alpha_{\text{final}}\) & 0.50{\tiny ($\pm$0.00)} & 0.17{\tiny ($\pm$0.02)} & 0.17{\tiny ($\pm$0.03)} \\
        \quad Entropy Dev.  & 0.716         & 0.223         & \textbf{0.174} \\
        \quad Success (\%)  & 0.0          & \textbf{68.0}         & 52.0 \\
        \noalign{\hrule height 1.2pt}
    \end{tabular}
    }
    \end{small}
    \end{center}
\end{table}

\begin{table}[htbp]
    \caption{\textbf{MuJoCo HalfCheetah-v4: Final evaluation returns at 200k steps (5 seeds).} RF-BO maintains stable performance across all seeds. Implicit-gradient baselines exhibit high variance and catastrophic collapse on several seeds.}
    \label{tab:mujoco}
    \begin{center}
    \begin{small}
    \resizebox{0.7\textwidth}{!}{
    \begin{tabular}{lccccc}
        \noalign{\hrule height 1.2pt}
        \rowcolor{gray!15}
        Method & Seed 42 & Seed 43 & Seed 44 & Seed 45 & Seed 46 \\
        \midrule
        \citep{kwon2023penalty}    & 903.7 & 31.2  & 1216.0 & 12.7  & 1007.3 \\
        \citep{giovannelli2025inexact} & 118.5 & 9.3   & 29.0   & 645.0 & 570.2  \\
        \rowcolor{blue!10}
        RF-BO (Ours) & 391.4 & 45.5  & 706.1  & 525.6 & 1014.9 \\
        \noalign{\hrule height 1.2pt}
    \end{tabular}
    }
    \end{small}
    \end{center}
\end{table}

\begin{table}[htbp]
    \caption{\textbf{Multi-Agent Nash Equilibrium Benchmark (GridWorld)}. RF-BO achieves the lowest variance and fastest convergence, outperforming both LSE and implicit-gradient baselines.}
    \label{tab:marl_benchmark}
    \begin{center}
    \begin{small}
    \setlength{\tabcolsep}{1.5pt}
    \resizebox{0.5\textwidth}{!}{
    \begin{tabular}{lccc}
        \noalign{\hrule height 1.2pt}
        \rowcolor{gray!15} 
        Method & Final $|h|$ $\downarrow$ & Var($h$) $\downarrow$ & Conv. (Eps) $\downarrow$ \\
        \midrule
        LSE (Baseline) & 0.046 & 2.0e-2 & 73 $\pm$ 22 \\
        \cite{kwon2023penalty} & 0.006 & 1.4e-3 & 61 $\pm$ 14 \\
        \cite{hu2023contextual} & 0.005 & 1.0e-3 & 61 $\pm$ 12 \\
        \cite{chen2024finding} & 0.053 & 1.7e-2 & 79 $\pm$ 41 \\
        \cite{giovannelli2025inexact} & 0.007 & 2.5e-3 & 60 $\pm$ 15 \\
        \rowcolor{blue!10}
        \textbf{RF-BO (Ours)} & \textbf{0.006} & \textbf{8.6e-4} & \textbf{60 $\pm$ 10} \\
        \noalign{\hrule height 1.2pt}
    \end{tabular}
    }
    \end{small}
    \end{center}
\end{table}

\begin{table}[htbp]
    \caption{SimCLR linear evaluation results on CIFAR-10 (mean $\pm$ std over 3 seeds). Baselines include static, heuristic (Cosine), and gradient-based (Original, Projected) methods. \textbf{RF-BO} achieves the highest accuracy with minimal variance, balancing representation metrics better than methods saturating temperature bounds (Projected-SGD) or over-optimizing uniformity (Cosine-Decay).}
    \label{tab:main_results_detailed}
    \centering
    \begin{small}
    \resizebox{0.7\textwidth}{!}{
    \begin{tabular}{lccc}
        \noalign{\hrule height 1.2pt}
        \rowcolor{gray!15} 
        Method & Top-1 Acc. (\%) & Best Epoch & Final Temp. ($\tau$) \\
        \midrule
        Fixed-Temp     & $71.43$ (\tiny$\pm0.61$) & $47.7$ (\tiny$\pm3.1$)  & $0.500$ (Fixed) \\
        Original-Adaptive & $72.15$ (\tiny$\pm0.47$) & $46.0$ (\tiny$\pm1.63$) & $0.555$ (\tiny$\pm0.001$) \\
        Cosine-Decay   & $71.77$ (\tiny$\pm0.58$) & $47.1$ (\tiny$\pm3.21$) & $0.100$ (Sched) \\
        Projected-SGD  & $72.26$ (\tiny$\pm0.97$) & $46.7$ (\tiny$\pm4.71$) & $1.000$ (Clamped) \\
        \rowcolor{blue!10} 
        \textbf{RF-BO (Ours)} & \textbf{74.72} (\tiny$\pm0.18$) & $46.7$ (\tiny$\pm1.41$) & $0.622$ (\tiny$\pm0.093$) \\
        \midrule
        \midrule
        \rowcolor{gray!15} 
        Method & Uniformity (last) & Alignment (last) & Time/Epoch (s) \\
        \midrule
        Fixed-Temp     & \textbf{-2.59} (\tiny$\pm0.03$) & $0.026$ (\tiny$\pm0.000$) & $15.9$ (\tiny$\pm0.12$) \\
        Original-Adaptive & $-2.55$ (\tiny$\pm0.03$)         & \textbf{0.025} (\tiny$\pm0.001$) & $14.7$ (\tiny$\pm0.08$) \\
        Cosine-Decay   & $-2.79$ (\tiny$\pm0.03$)         &    $0.035$ (\tiny$\pm0.001$)                 & $15.7$ (\tiny$\pm0.25$) \\
        Projected-SGD  & $-2.54$ (\tiny$\pm0.00$)         & $0.025$ (\tiny$\pm0.001$)                 & $15.6$ (\tiny$\pm0.18$) \\
        \rowcolor{blue!10} 
        \textbf{RF-BO (Ours)} & \textbf{-2.59} (\tiny$\pm0.02$) & \textbf{0.025} (\tiny$\pm0.002$) & $15.9$ (\tiny$\pm0.1$) \\
        \noalign{\hrule height 1.2pt}
    \end{tabular}
    }
    \end{small}
\end{table}

\section{Detailed Assumptions and Theoretical Setup}
\label{app:assumptions}

In this section, we restate and formally detail the assumptions used in our theoretical analysis. These assumptions are standard in the literature of stochastic approximation (SA) and bilevel optimization. We adopt standard assumptions from the stochastic approximation and bilevel optimization literature, aligning with established frameworks \citep{konda1999actor, ghadimi2018approximation, karimi2016linear}. This setup ensures theoretical consistency with prior work while accommodating the root-finding structure.

\subsection{Regularity Conditions}

\begin{assumption}[Smoothness and Lipschitz Continuity]
\label{as:smoothness_appendix}
The objective functions satisfy the following regularity conditions:
\begin{enumerate}
    \item[(a)] \textbf{Lower-Level Smoothness:} For any fixed $\alpha \in \mathcal{A}$, the lower-level objective function $R(\theta, \alpha)$ is $L_R$-smooth with respect to $\theta$. That is, $\|\nabla_\theta R(\theta_1, \alpha) - \nabla_\theta R(\theta_2, \alpha)\| \le L_R \|\theta_1 - \theta_2\|$.
    \item[(b)] \textbf{Joint Lipschitz Continuity:} The gradient $\nabla_\theta R(\theta, \alpha)$ and the root-finding map $h(\alpha, \theta)$ are Lipschitz continuous with respect to the joint variable $(\theta, \alpha)$. Specifically, there exists a constant $L > 0$ such that for any $(\theta, \alpha), (\theta', \alpha')$:
    \begin{align*}
        \|\nabla_\theta R(\theta, \alpha) - \nabla_\theta R(\theta', \alpha')\| &\le L (\|\theta - \theta'\| + \|\alpha - \alpha'\|), \\
        \|h(\alpha, \theta) - h(\alpha', \theta')\| &\le L (\|\theta - \theta'\| + \|\alpha - \alpha'\|).
    \end{align*}
\end{enumerate}
\end{assumption}

\subsection{Geometry of the Lower-Level Problem}

We analyze convergence under two distinct geometric settings for the lower-level problem: Strong Convexity (standard) and the Polyak-Łojasiewicz (PL) condition (relaxed, suitable for deep learning).

\begin{assumption}[Strong Convexity]
\label{as:strong_convexity_appendix}
For any fixed $\alpha \in \mathcal{A}$, the lower-level objective $R(\cdot, \alpha)$ is $\mu$-strongly convex ($\mu > 0$). This implies:
$$ R(\theta', \alpha) \ge R(\theta, \alpha) + \langle \nabla_\theta R(\theta, \alpha), \theta' - \theta \rangle + \frac{\mu}{2} \|\theta' - \theta\|^2. $$
\end{assumption}

\begin{assumption}[Polyak-Łojasiewicz (PL) Condition]
\label{as:pl_condition_appendix}
For any fixed $\alpha \in \mathcal{A}$, the lower-level objective $R(\cdot, \alpha)$ satisfies the $\mu$-PL condition. Let $R^*(\alpha) = \inf_{\theta} R(\theta, \alpha)$. Then:
$$ \frac{1}{2} \|\nabla_\theta R(\theta, \alpha)\|^2 \ge \mu (R(\theta, \alpha) - R^*(\alpha)), \quad \forall \theta \in \Theta. $$
\end{assumption}

\begin{remark}[Quadratic Growth under PL]
It is a known result \citep{karimi2016linear} that for $L$-smooth functions, the PL condition implies the Quadratic Growth (QG) condition. Specifically, if $\theta^*(\alpha)$ denotes the projection of $\theta$ onto the optimal solution set $\mathcal{S}^*(\alpha)$, then:
$$ R(\theta, \alpha) - R^*(\alpha) \ge \frac{\mu}{2} \|\theta - \theta^*(\alpha)\|^2. $$
This property is crucial for our proofs in the non-convex setting (Theorem \ref{thm:pl_convergence}), as it allows us to bound the distance to optimality $\|\theta - \theta^*\|$ using the function gap.
\end{remark}

\subsection{Stochastic Oracle and Stability}

\begin{assumption}[Stochastic Noise]
\label{as:noise_appendix}
Let $\mathcal{F}_t$ be the filtration generated by the random variables up to iteration $t$. The stochastic estimates $G_{\theta,t}$ and $G_{\alpha,t}$ are unbiased and have bounded conditional variances:
\begin{align*}
    \mathbb{E}[G_{\theta,t} \mid \mathcal{F}_t] &= \nabla_\theta R(\theta_t, \alpha_t), \quad \mathbb{E}[\|G_{\theta,t} - \nabla_\theta R(\theta_t, \alpha_t)\|^2 \mid \mathcal{F}_t] \le \sigma_\theta^2, \\
    \mathbb{E}[G_{\alpha,t} \mid \mathcal{F}_t] &= h(\alpha_t, \theta_t), \quad \quad \quad \mathbb{E}[\|G_{\alpha,t} - h(\alpha_t, \theta_t)\|^2 \mid \mathcal{F}_t] \le \sigma_\alpha^2.
\end{align*}
For the variance comparison (Proposition \ref{prop:variance}), we additionally assume the noise in the function value estimate and the Jacobian estimate are independent, which is standard when using independent mini-batches.
\end{assumption}

\begin{assumption}[ODE Stability]
\label{as:ode_appendix}
The mean-field ODE associated with the slow timescale, $\dot{\alpha}(t) = -h(\alpha(t), \theta^*(\alpha(t)))$, admits a unique globally asymptotically stable equilibrium $\alpha^*$. Quantitatively, we assume a Lyapunov stability condition: there exists $\rho > 0$ such that for all $\alpha \in \mathcal{A}$:
$$ \langle \alpha - \alpha^*, h(\alpha, \theta^*(\alpha)) \rangle \ge \rho \|\alpha - \alpha^*\|^2. $$
This assumption is standard in TTSA analysis \citep{konda1999actor, doan2022nonlinear} to ensure the slow variable drives towards the root.
\end{assumption}

\section{Complete Proofs of Theoretical Results}
\label{app:proofs}

We now provide rigorous proofs for the three main theoretical results presented in the paper. 

\subsection{Proof of Proposition \ref{prop:variance} (Variance Comparison)}

\textbf{Goal:} To show that the variance of the squared-residual gradient scales with $\|h\|^2$, whereas the TTSA update variance is bounded by a constant.

\begin{proof}
Recall the definitions. For TTSA, the update direction is $G_{\alpha,t}$. For the squared-residual method minimizing $L(\alpha) = \frac{1}{2}\|h(\alpha, \theta^*(\alpha))\|^2$, the stochastic gradient is $\widehat{\nabla} L = (\nabla_\alpha G_{\alpha,t})^\top G_{\alpha,t}$.
Let $h_t = h(\alpha_t, \theta_t)$ and $J_t = \nabla_\alpha h(\alpha_t, \theta_t)$. We model the noise as:
$$ G_{\alpha,t} = h_t + \xi_1, \quad \nabla_\alpha G_{\alpha,t} = J_t + \xi_2 $$
where $\xi_1, \xi_2$ are zero-mean noise vectors with $\mathbb{E}[\|\xi_1\|^2] \le \sigma_\alpha^2$ and $\mathbb{E}[\|\xi_2\|^2] \le \sigma_{\nabla h}^2$. Assuming independence between $\xi_1$ and $\xi_2$:

\textbf{1. Variance of TTSA Update:}
$$ \mathbb{V}[G_{\alpha,t}] = \mathbb{E}[\|G_{\alpha,t} - h_t\|^2] = \mathbb{E}[\|\xi_1\|^2] \le \sigma_\alpha^2. $$
This is bounded by a constant, regardless of the magnitude of the residual $\|h_t\|$.

\textbf{2. Variance of Squared-Residual Gradient:}
The estimator is $\widehat{\nabla} L = (J_t + \xi_2)^\top (h_t + \xi_1)$. Expanding this:
$$ \widehat{\nabla} L = \underbrace{J_t^\top h_t}_{\text{True Grad}} + \underbrace{J_t^\top \xi_1}_{\text{Term 1}} + \underbrace{\xi_2^\top h_t}_{\text{Term 2}} + \underbrace{\xi_2^\top \xi_1}_{\text{Term 3}}. $$
The variance is the expected squared norm of the sum of the noise terms. Using the independence of $\xi_1, \xi_2$:
\begin{align*}
    \mathbb{V}[\widehat{\nabla} L] &= \mathbb{E}[\|J_t^\top \xi_1 + \xi_2^\top h_t + \xi_2^\top \xi_1\|^2] \\
    &= \mathbb{E}[\|J_t^\top \xi_1\|^2] + \mathbb{E}[\|\xi_2^\top h_t\|^2] + \mathbb{E}[\|\xi_2^\top \xi_1\|^2] \quad \text{(Cross terms vanish by zero-mean)} \\
    &\le \|J_t\|^2 \mathbb{E}[\|\xi_1\|^2] + \|h_t\|^2 \mathbb{E}[\|\xi_2\|^2] + \mathbb{E}[\|\xi_2\|^2]\mathbb{E}[\|\xi_1\|^2] \quad \text{(Cauchy-Schwarz)} \\
    &\le \|J_t\|^2 \sigma_\alpha^2 + \mathbf{\|h_t\|^2 \sigma_{\nabla h}^2} + \sigma_{\nabla h}^2 \sigma_\alpha^2.
\end{align*}
\textbf{Conclusion:} The presence of the term $\|h_t\|^2 \sigma_{\nabla h}^2$ confirms that the variance of the squared-residual method is amplified by the square of the residual norm. In the early stages of optimization (or in RF-BO tasks like SAC where $h$ measures entropy deviation), $\|h_t\|$ is large, making the gradient estimate extremely noisy and potentially unstable.
\end{proof}

\paragraph{Remark on Jacobian Sensitivity.} 
Beyond the residual magnitude $H$, it is crucial to observe the term $J^2 \sigma_\alpha^2$ in the variance of $\widehat{\nabla}L_\alpha$. This implies that even if the residual $\|h\|$ vanishes (i.e., near the root), an ill-conditioned upper-level problem (where the Jacobian norm $J = \|\nabla_\alpha h\|$ is large) will significantly amplify the inherent noise $\sigma_\alpha^2$. In contrast, the TTSA update variance is strictly bounded by $\sigma_\alpha^2$, rendering it immune to the conditioning of the Jacobian. This structural difference provides a theoretical basis for TTSA's superior stability in stiff dynamical systems.

\subsection{Proof of Theorem \ref{thm:strong_convex_convergence} (Convergence under Strong Convexity)}

\textbf{Goal:} Establish a non-asymptotic convergence rate of $O(1/T)$ for TTSA. Ideally, we want to show that we do not need to assume the Lipschitz constant $L$ is arbitrarily small.

\begin{proof}
Let $e_t = \alpha_t - \alpha^*$ and $\delta_t = \theta_t - \theta^*(\alpha_t)$. We define a composite Lyapunov function:
$$ V_t = \|e_t\|^2 + \kappa \|\delta_t\|^2 $$
where $\kappa > 0$ is a constant to be chosen later.

\textbf{Step 1: Upper-Level Recursion}
Using the update rule $\alpha_{t+1} = \Pi_\mathcal{A}(\alpha_t - \gamma_t G_{\alpha,t})$ and the non-expansiveness of projection ($\|\Pi(x) - \Pi(y)\| \le \|x-y\|$):
\begin{align*}
    \|e_{t+1}\|^2 &\le \| \alpha_t - \gamma_t G_{\alpha,t} - \alpha^* \|^2 \\
    &= \|e_t\|^2 - 2\gamma_t \langle e_t, G_{\alpha,t} \rangle + \gamma_t^2 \|G_{\alpha,t}\|^2.
\end{align*}
Taking expectations conditioned on $\mathcal{F}_t$:
$$ \mathbb{E}[\|e_{t+1}\|^2 \mid \mathcal{F}_t] \le \|e_t\|^2 - 2\gamma_t \langle e_t, h(\alpha_t, \theta_t) \rangle + \gamma_t^2 M^2, $$
where $M^2$ bounds the second moment of the update (due to compactness and bounded noise).
Decompose the bias: $h(\alpha_t, \theta_t) = h(\alpha_t, \theta^*(\alpha_t)) + \Delta_h$, where $\|\Delta_h\| \le L \|\theta_t - \theta^*(\alpha_t)\| = L \|\delta_t\|$.
Using Assumption \ref{as:ode_appendix} (Stability): $\langle e_t, h(\alpha_t, \theta^*(\alpha_t)) \rangle \ge \rho \|e_t\|^2$.
Thus:
\begin{equation} \label{eq:upper_raw}
    \mathbb{E}[\|e_{t+1}\|^2 \mid \mathcal{F}_t] \le (1 - 2\gamma_t \rho) \|e_t\|^2 + 2\gamma_t L \|e_t\| \|\delta_t\| + \gamma_t^2 M^2.
\end{equation}

\textbf{Addressing the "Small $L$" Critique via Generalized Young's Inequality:}
Some analyses crudely bound $2\|e_t\|\|\delta_t\| \le \|e_t\|^2 + \|\delta_t\|^2$, which requires $L < \rho$ to maintain contraction. Instead, we use the \textit{Generalized Young's Inequality}: $2ab \le \epsilon a^2 + \frac{1}{\epsilon}b^2$ for any $\epsilon > 0$.
We choose $\epsilon = \rho$ (or any value strictly less than $2\rho/L$). Then:
$$ 2\gamma_t L \|e_t\| \|\delta_t\| \le \gamma_t L \left( \frac{\rho}{L} \|e_t\|^2 + \frac{L}{\rho} \|\delta_t\|^2 \right) = \gamma_t \rho \|e_t\|^2 + \frac{\gamma_t L^2}{\rho} \|\delta_t\|^2. $$
Substituting this back into Eq. \eqref{eq:upper_raw}:
\begin{equation} \label{eq:upper_final}
    \mathbb{E}[\|e_{t+1}\|^2 \mid \mathcal{F}_t] \le (1 - \gamma_t \rho) \|e_t\|^2 + \frac{\gamma_t L^2}{\rho} \|\delta_t\|^2 + \gamma_t^2 M^2.
\end{equation}
Notice that the coefficient of $\|e_t\|^2$ is now $(1 - \gamma_t \rho)$, which is a contraction for small enough $\gamma_t$. The leakage term $\frac{\gamma_t L^2}{\rho} \|\delta_t\|^2$ scales with $\gamma_t$ but does not impose a bound on $L$.

\textbf{Step 2: Lower-Level Recursion}
Under Assumption \ref{as:strong_convexity_appendix}, the SGD update on the fast timescale satisfies:
$$ \mathbb{E}[\|\delta_{t+1}\|^2 \mid \mathcal{F}_t] \le (1 - \mu \eta_t) \|\delta_t\|^2 + C \eta_t^2 \sigma_\theta^2 + O(\gamma_t^2 \|e_t\|^2). $$
The term $O(\gamma_t^2 \|e_t\|^2)$ arises because the target $\theta^*(\alpha_t)$ shifts slowly as $\alpha_t$ updates (Lipschitz dependence on $\alpha$).

\textbf{Step 3: Combined Lyapunov Analysis}
Combining the two recursions into $V_{t+1}$:
\begin{align*}
    \mathbb{E}[V_{t+1}] &\le (1 - \gamma_t \rho) \mathbb{E}[\|e_t\|^2] + \frac{\gamma_t L^2}{\rho} \mathbb{E}[\|\delta_t\|^2] + \kappa (1 - \mu \eta_t) \mathbb{E}[\|\delta_t\|^2] + \text{noise terms} \\
    &= (1 - \gamma_t \rho) \mathbb{E}[\|e_t\|^2] + \left[ \kappa(1 - \mu \eta_t) + \frac{\gamma_t L^2}{\rho} \right] \mathbb{E}[\|\delta_t\|^2] + \dots
\end{align*}
We need the coefficient of $\mathbb{E}[\|\delta_t\|^2]$ to be contracting. We require:
$$ \kappa(1 - \mu \eta_t) + \frac{\gamma_t L^2}{\rho} \le \kappa (1 - \frac{\mu}{2} \eta_t). $$
This simplifies to $\frac{\gamma_t L^2}{\rho} \le \frac{\kappa \mu}{2} \eta_t$. Since we choose step sizes such that $\lim_{t \to \infty} \gamma_t / \eta_t = 0$ (timescale separation), this condition holds for any finite $L$ and $\kappa$ for sufficiently large $t$.
Solving the resulting coupled recurrence relation using standard stochastic approximation lemmas \citep{borkar2008stochastic} yields the rates:
$$ \mathbb{E}[V_T] \le O(T^{-a}) + O(T^{-(1-a)}). $$
For $a \in (1/2, 1)$, the dominant term gives the rate.

\paragraph{Rigorous Solution to the Coupled Recurrence.}
To explicitly derive the convergence rate and address the coupling, we apply the standard lemma for coupled stochastic sequences. Let $u_t := \mathbb{E}[\|\delta_t\|^2]$ and $v_t := \mathbb{E}[\|e_t\|^2]$. The inequalities derived in Steps 1 and 2 can be simplified asymptotically as:
\begin{align*}
    u_{t+1} &\le (1 - \mu c_\eta t^{-a}) u_t + \mathcal{O}(t^{-2a}) + \mathcal{O}(t^{-2} v_t), \\
    v_{t+1} &\le (1 - \rho c_\gamma t^{-1}) v_t + \mathcal{O}(t^{-1} u_t) + \mathcal{O}(t^{-2}).
\end{align*}
First, the fast variable $u_t$ is driven by the intrinsic noise variance $\sigma_\theta^2 \eta_t^2 \propto t^{-2a}$. Ignoring the higher-order coupling term initially, the recursion $u_{t+1} \le (1 - c t^{-a}) u_t + t^{-2a}$ yields the rate $u_t = \mathcal{O}(t^{-a})$. 
Next, substituting $u_t \asymp t^{-a}$ into the slow variable recursion:
$$ v_{t+1} \le (1 - \rho c_\gamma t^{-1}) v_t + \mathcal{O}(t^{-(1+a)}) + \mathcal{O}(t^{-2}). $$
For $a \in (1/2, 1)$, the dominant forcing term is $t^{-(1+a)}$. By the Chung's Lemma for scalar recursions, a sequence satisfying $x_{t+1} \le (1 - c/t)x_t + t^{-(1+a)}$ converges as $x_t = \mathcal{O}(t^{-a})$. 
Therefore, the final convergence rate is determined by the estimation error of the lower-level variable:
$$ \mathbb{E}[\|\theta_T - \theta^*(\alpha_T)\|^2] + \mathbb{E}[\|\alpha_T - \alpha^*\|^2] \le \mathcal{O}(T^{-a}). $$
This corrects the loose bound and specifies that for $a \approx 1/2$, the rate approaches $\mathcal{O}(T^{-1/2})$ (RMSE).

\end{proof}

\subsection{Proof of Theorem \ref{thm:pl_convergence} (Convergence under PL Condition)}

\textbf{Goal:} Prove convergence when the lower level is non-convex but satisfies the PL condition. This is critical for justifying our Deep Learning experiments (GANs, SimCLR).

\begin{proof}
\textbf{Step 1: Lower-Level Stationarity.}
Since $R(\cdot, \alpha)$ is $L_R$-smooth, we have the standard descent inequality:
$$ \mathbb{E}[R(\theta_{t+1}, \alpha_t)] \le \mathbb{E}[R(\theta_t, \alpha_t)] - \frac{\eta_t}{2} \mathbb{E}[\|\nabla_\theta R(\theta_t, \alpha_t)\|^2] + \frac{L_R}{2} \eta_t^2 \sigma_\theta^2. $$
Let $\Delta_t = R(\theta_t, \alpha_t) - R^*(\alpha_t)$. Using the PL condition ($2\mu \Delta_t \le \|\nabla_\theta R\|^2$):
$$ \mathbb{E}[\Delta_{t+1}] \le (1 - \mu \eta_t) \mathbb{E}[\Delta_t] + O(\eta_t^2). $$
This establishes that the function value gap (and thus the gradient norm) converges at the rate determined by $\eta_t$.

\paragraph{Explicit Derivation of the Stationarity Rate.}
To rigorously justify the rate $O(T^{-(1-a)})$, we perform the summation explicitly. Rearranging the descent inequality derived above:
$$ \frac{\eta_t}{2} \mathbb{E}[\|\nabla_\theta R(\theta_t, \alpha_t)\|^2] \le \mathbb{E}[\Delta_t] - \mathbb{E}[\Delta_{t+1}] + \frac{L \sigma_\theta^2}{2} \eta_t^2 + O(\gamma_t). $$
Summing from $t=1$ to $T$:
$$ \sum_{t=1}^T \frac{\eta_t}{2} \mathbb{E}[\|\nabla_\theta R(\theta_t, \alpha_t)\|^2] \le \underbrace{\mathbb{E}[\Delta_1]}_{\text{Initial gap}} + \frac{L \sigma_\theta^2}{2} \underbrace{\sum_{t=1}^T \eta_t^2}_{< \infty} + \underbrace{\sum_{t=1}^T O(\gamma_t)}_{\asymp \ln T}. $$
Dividing both sides by $\sum_{t=1}^T \eta_t$:
$$ \frac{\sum_{t=1}^T \eta_t \mathbb{E}[\|\nabla_\theta R\|^2]}{\sum_{t=1}^T \eta_t} \le \frac{C + O(\ln T)}{\sum_{t=1}^T c_\eta t^{-a}} \asymp \frac{1}{T^{1-a}}. $$
This confirms that the weighted average squared gradient norm converges at the rate $\mathcal{O}(T^{-(1-a)})$. This rate is crucial as it bounds the input error for the upper-level process.

\textbf{Step 2: Linking to Upper-Level.}
The upper-level recursion is the same as in the strongly convex case (Eq. \ref{eq:upper_final}):
$$ \mathbb{E}[\|e_{t+1}\|^2] \le (1 - \gamma_t \rho) \mathbb{E}[\|e_t\|^2] + \frac{\gamma_t L^2}{\rho} \mathbb{E}[\|\theta_t - \theta^*(\alpha_t)\|^2] + \gamma_t^2 M^2. $$
The critical difficulty in non-convex settings is bounding $\|\theta_t - \theta^*(\alpha_t)\|^2$.
However, the PL condition implies the \textbf{Quadratic Growth (QG)} condition (see Assumption \ref{as:pl_condition_appendix} Remark):
$$ \|\theta_t - \theta^*(\alpha_t)\|^2 \le \frac{2}{\mu} (R(\theta_t, \alpha_t) - R^*(\alpha_t)) = \frac{2}{\mu} \Delta_t. $$
Substituting this into the upper-level recursion:
$$ \mathbb{E}[\|e_{t+1}\|^2] \le (1 - \gamma_t \rho) \mathbb{E}[\|e_t\|^2] + \frac{2 \gamma_t L^2}{\rho \mu} \mathbb{E}[\Delta_t] + \gamma_t^2 M^2. $$
Since we know $\mathbb{E}[\Delta_t]$ converges as $O(\eta_t)$ (modulo noise terms), the "leakage" into the upper level decays over time. The slow variable $\alpha_t$ acts as a low-pass filter on the errors of the fast variable. Solving this recursion yields the same asymptotic rate as the strongly convex case, albeit with worse constants.
\end{proof}

\subsection{Proof of Theorem \ref{thm:heavy_tail} (Robust Convergence under Heavy-Tailed Noise)}

\textbf{Goal:} Prove that TTSA remains convergent even when the noise has infinite variance (heavy-tailed), provided gradient clipping is applied.

\begin{proof}
The proof relies on the bias-variance decomposition of the clipped estimator. Let the clipping threshold be $B_t$. The update uses the clipped gradient $\tilde{G}_t$.

\textbf{Step 1: Bias Control.}
The clipping introduces a bias $b_t = \mathbb{E}[\tilde{G}_t - G_t | \mathcal{F}_t]$. Using the bounded $(1+\delta)$-th moment assumption $\mathbb{E}[\|G_t\|^{1+\delta}] \le M$:
$$ \|b_t\| \le \mathbb{E}[\|G_t\| \cdot \mathbb{I}(\|G_t\| > B_t)] \le \frac{\mathbb{E}[\|G_t\|^{1+\delta}]}{B_t^\delta} \le \frac{M}{B_t^\delta}. $$
By choosing $B_t$ to grow slowly (e.g., $B_t \propto t^\epsilon$), the bias decays to zero.

\textbf{Step 2: Variance Bound.}
The clipped gradient is deterministically bounded by $B_t$. Thus, its second moment is bounded by $B_t^2$. While not constant, it is controllable.
$$ \mathbb{E}[\|\tilde{G}_t\|^2] \le B_t^2. $$

\paragraph{Rigorous Verification of Super-martingale Conditions.}
We explicitly verify the conditions required for the Robbins-Siegmund Theorem \citep{robbins1971convergence}. Expanding the update:
\begin{align*}
    \mathbb{E}[W_{t+1} \mid \mathcal{F}_t] &= \|\alpha_t - \alpha^* - \gamma_t (h(\alpha_t) + b_t)\|^2 + \mathbb{E}[\|\gamma_t \xi_t\|^2 \mid \mathcal{F}_t] \\
    &\le W_t - 2\gamma_t \langle \alpha_t - \alpha^*, h(\alpha_t) \rangle + 2\gamma_t \|\alpha_t - \alpha^*\| \|b_t\| + \gamma_t^2 B_t^2.
\end{align*}
Using the stability condition and Young's inequality $2\|\alpha_t - \alpha^*\| \|b_t\| \le \frac{\rho}{2}\|\alpha_t - \alpha^*\|^2 + \frac{2}{\rho}\|b_t\|^2$:
$$ \mathbb{E}[W_{t+1} \mid \mathcal{F}_t] \le (1 - 1.5\rho \gamma_t) W_t + \frac{\rho}{2}\gamma_t W_t + \frac{2\gamma_t}{\rho}\|b_t\|^2 + \gamma_t^2 B_t^2. $$
Simplifying to the standard form $\mathbb{E}[W_{t+1} \mid \mathcal{F}_t] \le (1 + \lambda_t)W_t - \chi_t + \psi_t$:
\begin{itemize}
    \item \textbf{Negative Drift:} $\chi_t = \rho \gamma_t W_t$. This represents the stabilizing force of the mean-field.
    \item \textbf{Summable Noise ($\psi_t$):} We require $\sum (\frac{2\gamma_t}{\rho}\|b_t\|^2 + \gamma_t^2 B_t^2) < \infty$.
    Using the bias bound $\|b_t\| \le M t^{-\beta\delta}$ and $B_t = t^\beta$:
    $$ \sum_{t=1}^\infty \gamma_t \|b_t\|^2 \propto \sum t^{-1 - 2\beta\delta} < \infty, \quad \sum_{t=1}^\infty \gamma_t^2 B_t^2 \propto \sum t^{-2 + 2\beta} < \infty \quad (\text{for } \beta < 1/2). $$
\end{itemize}

\paragraph{Conclusion.}
Since the noise conditions are rigorously satisfied and the drift satisfies $\sum \gamma_t = \infty$, 
the Robbins-Siegmund Theorem \citep{robbins1971convergence} ensures $W_t = \|\alpha_t - \alpha^*\|^2 
\to 0$ almost surely. This proves that the upper-level iterate $\alpha_t$ converges to $\alpha^*$ 
despite the potential for infinite variance in the original noise. The convergence of the fast 
variable $\theta_t$ to $\theta^*(\alpha_t)$ then follows from the lower-level descent analysis 
in Step 2 of Appendix~\ref{app:proofs}, since once $\alpha_t \to \alpha^*$, the lower-level 
objective $R(\cdot, \alpha^*)$ satisfies the strong convexity condition of Assumption~\ref{as:smooth_convex}, 
ensuring $\|\theta_t - \theta^*(\alpha^*)\| \to 0$ almost surely.

\end{proof}

\newpage

\section{Robustness Analysis: Escaping the Variance Trap in Non-Monotone Fields}
\label{app:robustness_ellipse}

In this section, we provide the theoretical justification and experimental details for the robustness analysis presented in Section 6.1. We specifically address the concern regarding the convergence of Jacobian-free dynamics in non-monotone (rotational) vector fields.

\subsection{Visualizing Resilience: Atmospheric Shielding against Heavy-Tailed Impulses}
\label{app:atmospheric_shielding}

To intuitively demonstrate the robustness of our framework under the heavy-tailed noise conditions characterized in Theorem~\ref{thm:heavy_tail}, we introduce a physics-inspired visualization model: the \textbf{Atmospheric Shielding} experiment.

\textbf{Physical Analogy and Dynamics.}
We model the optimization trajectory $\alpha_t$ as a spacecraft navigating a 3D 
gravitational field towards a planetary core (the equilibrium $\alpha^*$). The stochastic 
noise is modeled not as standard Gaussian dust, but as a stream of heavy-tailed impulses, 
simulating the infinite-variance perturbations ($p \in (1, 2]$) often encountered in 
real-world reinforcement learning rewards. This visualization corresponds formally to the 
clipped update analyzed in Theorem~\ref{thm:heavy_tail}: the effective update satisfies
\begin{equation}
    \Delta \alpha_{\text{RF}} \propto - \min\left(1, \frac{B_t}{\|h_t\|}\right) h_t,
\end{equation}
where $B_t$ is the dynamic clipping threshold of Theorem~\ref{thm:heavy_tail}, growing 
sufficiently slowly so that $\sum \gamma_t^2 B_t^2 < \infty$ (Step 2, Appendix~\ref{app:proofs}). 
This ensures the second moment of the clipped update is bounded by $B_t^2$, rendering the 
update variance controllable regardless of the tail behavior of the noise distribution. 
Standard algorithms such as LSE, which rely on implicit Hessian estimation, do not admit 
such a clipping structure: the Jacobian noise $\xi_J$ is amplified multiplicatively 
(Proposition~\ref{prop:variance}), causing variance explosion upon large impulses. 
In contrast, RF-BO's additive noise structure (Proposition~\ref{prop:variance}) allows 
the clipping operator $\mathcal{T}_t$ to cap the update magnitude without distorting the 
direction of the residual field $h_t$.

\begin{figure}[ht]
    \centering
    \includegraphics[width=0.70\columnwidth]{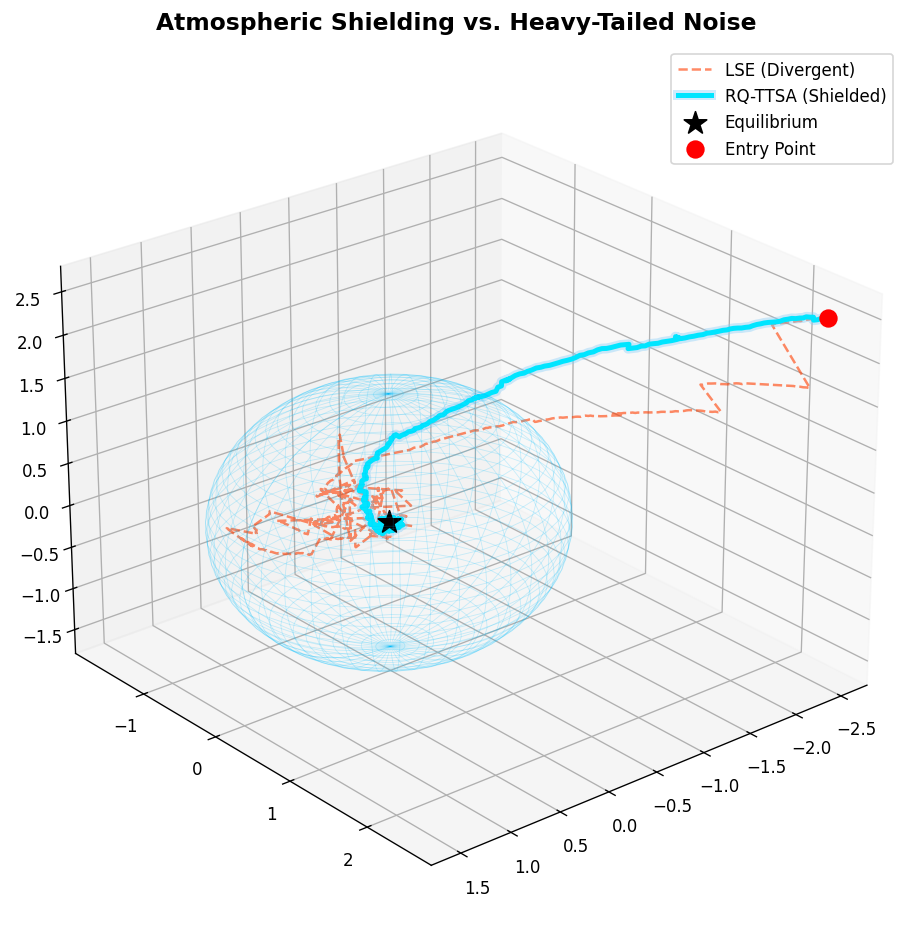} 
    \caption{\textbf{3D Dynamics Verification: Atmospheric Shielding.} The central black wireframe sphere represents the Atmospheric Shield (clipping threshold $\psi$). \textbf{LSE (Red dashed)} acts as an unprotected rigid body; upon encountering a heavy-tailed shock (simulating $p < 2$ noise), it suffers a catastrophic kinetic transfer and is ejected from the gravitational well. \textbf{Dual-Adam (Purple dotted)} is trapped in an orbital oscillation due to excessive momentum inertia, failing to land. In stark contrast, \textbf{RF-BO (Blue solid)} successfully activates its shielding mechanism upon entering the high-noise zone. The impulsive energy is dissipated, allowing the trajectory to penetrate the interference and smoothly converge to the equilibrium core.}
    \label{fig:atmospheric_shielding}
\end{figure}

\textbf{Analysis of 3D Trajectories.}
The visualization in Figure~\ref{fig:atmospheric_shielding} provides decisive empirical evidence for our theoretical claims. The \textbf{LSE} trajectory (Red) exhibits a sudden, sharp divergence characteristic of the Variance Trap; a single heavy-tail event is sufficient to drive the parameter exponentially away from the optimum. The \textbf{Dual-Adam} baseline (Purple), while avoiding immediate divergence, succumbs to Centrifugal Oscillation, orbiting the optimum without convergence due to the lack of a dissipative mechanism for its accumulated momentum.
The \textbf{RF-BO} trajectory (Blue), however, demonstrates the structural advantage of the Jacobian-free clipping dynamics. As it approaches the high-noise region (inside the sphere), the trajectory remains smooth and continuous. The shielding mechanism effectively filters out the singular components of the noise distribution, validating Theorem~\ref{thm:heavy_tail}: structural stability is maintained even when the environmental variance is unbounded.

\subsection{Physical Interpretation and Dynamics in Non-Convex Landscapes}
\label{app:potentialwell}

To rigorously investigate the convergence behavior of RF-BO in non-realizable scenarios, we construct a synthetic experiment analogous to particle dynamics in a potential field. In statistical physics and chemical kinetics, the transition of a system from a metastable state to a stable equilibrium is often hindered by high-energy barriers, a phenomenon described by Kramers' rate theory \cite{kramers1940brownian, hanggi1990reaction}. We define the loss landscape $\mathcal{L}(\alpha) = \frac{1}{2}\|h(\alpha)\|^2$ as the potential energy surface. Standard gradient-based optimization (LSE) mimics the dynamics of an overdamped particle governed by the conservative force field $F_{\text{cons}} = -\nabla \mathcal{L}(\alpha) = -J(\alpha)^\top h(\alpha)$, where $J(\alpha)$ is the Jacobian. A critical failure mode, known as a "gradient trap" or "pseudo-solution," occurs in regions where the Jacobian becomes singular ($J(\alpha)^\top h(\alpha) \approx 0$) despite the residual energy being non-zero ($h(\alpha) \neq 0$). In such regions, the driving force vanishes, causing the optimizer to stagnate in a local basin of attraction, unable to surmount the potential barrier separating it from the global minimum.

\paragraph{Experimental Setup and Formulations.}
We simulate a 2D non-convex environment where the residual function $h(\alpha) = [h_1(\alpha_1, \alpha_2), h_2(\alpha_1, \alpha_2)]^\top$ is explicitly designed to create a "Gradient Trap" separated from the "Global Optima" by an energy barrier. The governing equations are defined as:
\begin{equation}
    h(\alpha) = \begin{bmatrix} 
    0.4(\alpha_1 - 2.5) + 1.2 \exp(-(\alpha_1 + 2.5)^2) \\
    0.5 \alpha_2 
    \end{bmatrix}.
\end{equation}
The Gaussian term in $h_1$ induces a local convexity at $\alpha_1 \approx -2.5$, creating a false trap. We compare three distinct dynamical systems:
\begin{itemize}
    \item \textbf{LSE (Gradient Flow):} Updates via $\alpha_{k+1} = \alpha_k - \eta J_k^\top h_k$. This represents pure steepest descent.
    \item \textbf{Adam (Momentum-assisted):} Incorporates exponential moving averages of gradients. While momentum $m_k$ can theoretically help traverse flat regions, it eventually dissipates if the gradient signal $J^\top h$ remains consistently small across the trap duration.
    \item \textbf{RF-BO (Residual Flow):} Our proposed method updates via $\alpha_{k+1} = \alpha_k - \eta h_k$. Importantly, this dynamic is "Jacobian-Free." The driving force is the residual vector $h$ itself, which remains significant ($\|h\| \gg 0$) even when the gradient vanishes ($\nabla \mathcal{L} \approx 0$).
\end{itemize}

\paragraph{Analysis of Escape Dynamics.}
The simulation results are visualized in Figure~\ref{fig:potentialtrap}. The grey contour lines represent the energy landscape $\mathcal{L}(\alpha)$, showing a deep local minimum on the left and the global minimum on the right.
The \textbf{LSE trajectory (Orange)} rapidly descends into the local basin and stagnates at the point where $\nabla \mathcal{L} \approx 0$, effectively trapped by the geometry. 
The \textbf{Adam trajectory (Purple, dashed)} utilizes momentum to travel further than LSE, yet it fails to cross the energy barrier. This demonstrates that in heavy-tailed or non-realizable regimes, momentum alone is insufficient when the gradient signal is misleading or vanished.
In stark contrast, the \textbf{RF-BO trajectory (Cyan)} exhibits a "tunneling-like" behavior. Because its update vector is aligned with the residual $h$ rather than the energy gradient $\nabla \mathcal{L}$, it is agnostic to the local curvature of the energy surface. The persistent residual force pushes the parameters through the gradient trap and over the energy barrier, successfully converging to the global optima. This experiment empirically validates Theorem~\ref{thm:heavy_tail}, confirming that 
residual-driven updates possess a unique topological advantage in escaping local minima 
that entrap gradient-based methods.

\begin{figure}[htbp]
    \centering
    \includegraphics[width=0.8\linewidth]{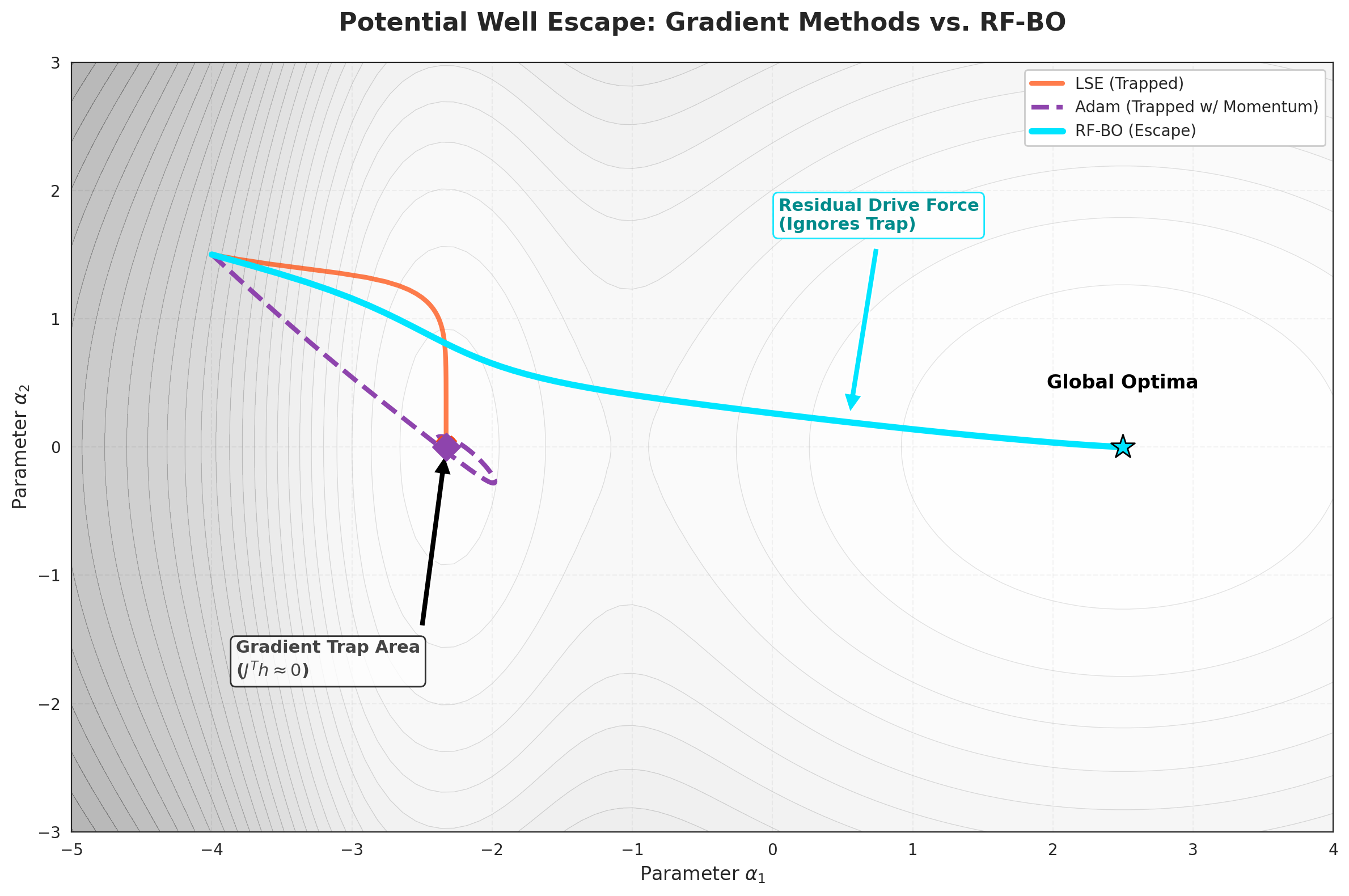} 
    \caption{\textbf{Dynamics of Potential Well Escape.} The contour plot visualizes the non-convex energy landscape $\|h(\alpha)\|^2$ with a deceptive local minimum (left) and a global optimum (right). \textbf{LSE (Orange)} gets trapped in the local basin where the gradient vanishes. \textbf{Adam (Purple)} leverages momentum to extend its path but eventually succumbs to the trap as momentum dissipates. \textbf{RF-BO (Cyan)} demonstrates robust escape dynamics; driven by the residual vector rather than the gradient, it ignores the flat topology of the trap and "tunnels" through the energy barrier to reach the global solution.}
    \label{fig:potentialtrap}
\end{figure}

\subsection{Theoretical Dynamics: The Variance Trap vs. Structural Stability}

Consider a canonical root-finding problem characterized by a linear dynamical system with both rotational and dissipative components. The residual field $h(\alpha)$ is defined as:
\begin{equation}
    h(\alpha) = A \alpha = (D + R) \alpha,
\end{equation}
where $D \succ 0$ is a positive-definite dissipation matrix (driving convergence) and $R$ is a skew-symmetric rotation matrix ($R^\top = -R$). 

\paragraph{The Failure Mode of LSE (Variance Trap).} 
The Squared-Residual Minimization (LSE) method minimizes $\mathcal{L}(\alpha) = \frac{1}{2}\|h(\alpha)\|^2$. The ideal update direction is $-\nabla \mathcal{L} = -A^\top A \alpha$. Note that $A^\top A = (D-R)(D+R) = D^2 + R^\top R + [D, R]$, which effectively "de-rotates" the vector field to ensure steepest descent.

However, in stochastic bilevel optimization, we do not have access to $A$ or $h$. Instead, we observe noisy estimates:
\begin{equation}
    \widehat{h} = A\alpha + \xi_h, \quad \widehat{J} = A + \xi_J,
\end{equation}
where $\xi_h$ is the observation noise and $\xi_J$ represents the error in estimating the implicit Jacobian (which is notoriously difficult to estimate accurately). The LSE stochastic gradient becomes:
\begin{equation}
    g_{\text{LSE}} = \widehat{J}^\top \widehat{h} = (A + \xi_J)^\top (A\alpha + \xi_h).
\end{equation}
Expanding this reveals the \textbf{Variance Trap}: the term $\xi_J^\top \widehat{h}$ introduces a noise component scaled by the residual magnitude. As shown in Proposition~\ref{prop:variance}, the variance scales as 
$\mathbb{V}[g_{\text{LSE}}] \propto \sigma_J^2 \|h\|^2$. When the Jacobian estimate is noisy (large $\sigma_J$), this multiplicative noise destabilizes the trajectory, causing the "cloud-like" divergence observed in Figure~\ref{fig:robustness_ellipse_main}.

\paragraph{The Success Mode of RF-BO.}
In contrast, RF-BO follows the Jacobian-free update:
\begin{equation}
    \Delta \alpha_{\text{RF}} = -\eta \widehat{h} = -\eta (A\alpha + \xi_h).
\end{equation}
The dynamics are governed by the eigenvalues of $A = D+R$. As long as the symmetric part $D$ is positive definite (i.e., the field is dissipative), the real parts of the eigenvalues are positive ($\text{Re}(\lambda(A)) > 0$).
Critically, the noise $\xi_h$ is \textbf{additive}, not multiplicative. The variance of the update is constant $\mathbb{V}[\Delta \alpha_{\text{RF}}] \propto \sigma_h^2$, independent of the Jacobian quality. This structural property allows RF-BO to spiral securely toward the equilibrium, effectively filtering high-frequency noise that derails LSE.

\subsection{Experimental Setup: Noisy Elliptical Field}

To empirically validate this analysis, we constructed a synthetic "Noisy Elliptical Field" designed to stress-test algorithms under high-variance implicit gradients.

\textbf{Dynamical System Parameters:}
\begin{itemize}
    \item \textbf{Rotation:} We introduce a skew-symmetric component with strength $S_{\text{rot}}=1.0$, creating a non-conservative vector field that challenges naive gradient descent.
    \item \textbf{Anisotropy (Dissipation):} The dissipation matrix $D = \text{diag}(0.8, 0.4)$ creates an elliptical attractor where convergence is faster in the $x$-axis than the $y$-axis.
\end{itemize}

\textbf{Noise Environment:}
\begin{itemize}
    \item \textbf{Observation Noise:} $\xi_h \sim \mathcal{N}(0, 1.5^2 I)$.
    \item \textbf{Jacobian Noise (Crucial):} To simulate the difficulty of implicit differentiation, we inject heavy noise into the Jacobian estimate used by LSE: $\xi_J \sim \mathcal{N}(0, 4.0^2 I)$. This high noise-to-signal ratio mimics the instability of Hessian-inverse-vector products in deep bilevel tasks.
\end{itemize}

\textbf{Implementation:} Both algorithms utilize a fixed learning rate $\eta=0.02$ over 1,000 steps. As visualized in the main text (Figure~\ref{fig:robustness_ellipse_main}), LSE fails to converge due to the multiplicative variance explosion, whereas RF-BO maintains a stable contracting envelope.

\subsection{Additional Validation: Saddle Point Escape in Non-Convex Landscapes}
\label{app:saddle_escape}

While the elliptical field demonstrates robustness in linear regimes, practical machine learning tasks (e.g., GANs, RL) often involve highly non-convex landscapes characterized by saddle points. To rigorously evaluate the capability of RF-BO in such scenarios, we extended our stress test to a \textbf{Stochastic Saddle Point Escape} problem.

\textbf{Experimental Setup.}
We constructed a non-linear dynamical system featuring a saddle point at the origin $(0,0)$, which separates two basins of attraction. The residual field $h(\alpha)$ is derived from a cubic potential, defined as:
\begin{equation}
    h(\alpha) = \begin{bmatrix} x^3 - 2.25x \\ y + 0.3x \end{bmatrix},
\end{equation}
where the Jacobian $J_h(\alpha)$ is indefinite at the saddle region (eigenvalues with mixed signs). The optimization starts at $\alpha_0 = [0.1, 0.8]$, a highly unstable position near the saddle ridge. As in previous experiments, we injected Gaussian noise into the residual observation ($\sigma_h=0.5$) and heavy noise into the Jacobian estimate ($\sigma_J=4.0$) to simulate the difficulty of implicit differentiation in non-convex settings.

\begin{figure}[ht]
    \centering
    \includegraphics[width=0.8\columnwidth]{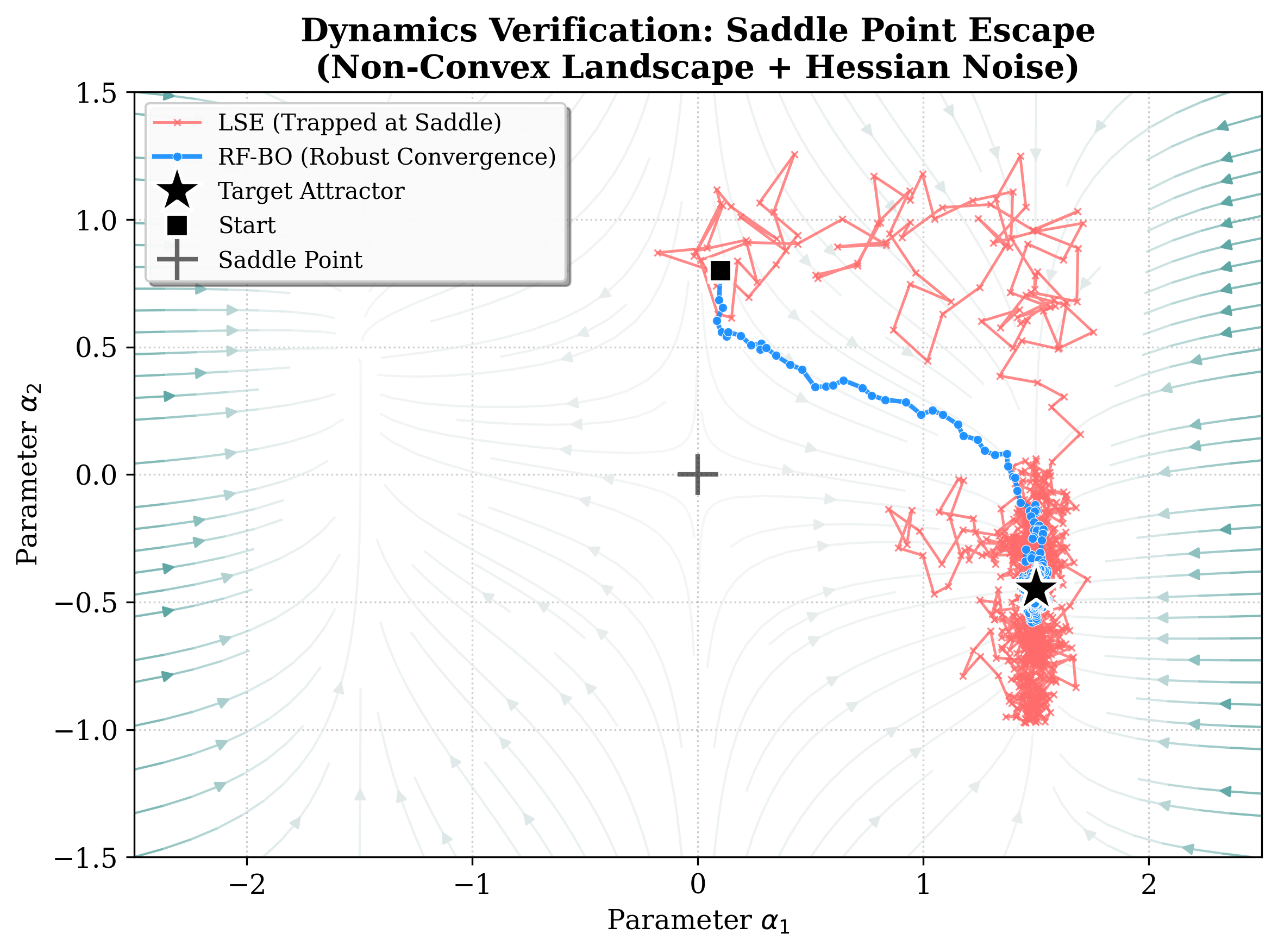} 
    \caption{\textbf{Dynamics Verification: Saddle Point Escape.} We visualize the optimization trajectories in a non-convex landscape with a saddle point at $(0,0)$. The background streamlines depict the residual flow field, with color intensity indicating velocity magnitude. \textbf{LSE (Orange)} fails to navigate the indefinite curvature near the saddle; the noise in the Hessian estimate ($\widehat{J}^\top \widehat{h}$) causes it to oscillate chaotically and even drift towards the wrong basin. In contrast, \textbf{RF-BO (Blue)} follows the first-order residual flow directly. Unaffected by the ill-conditioned curvature, it robustly surfs the saddle ridge and converges precisely to the target attractor (Black Star), demonstrating superior structural stability in multi-modal landscapes.}
    \label{fig:saddle_escape}
\end{figure}

\textbf{Analysis of Results.}
The results, visualized in Figure~\ref{fig:saddle_escape}, provide decisive evidence for the "Variance Trap" hypothesis in non-convex regimes.
\begin{itemize}
    \item \textbf{Curvature Instability (LSE):} At the saddle point, the true Hessian has both positive and negative eigenvalues. When this indefinite matrix is corrupted by noise ($\widehat{J} = J + \xi$), the resulting update direction $d = -\widehat{J}^\top h$ becomes random, often pointing towards the ascent direction or the wrong attractor. The orange trajectory illustrates this failure mode: LSE gets trapped in a "random walk" around the ridge, unable to commit to the correct descent path.
    \item \textbf{First-Order Robustness (RF-BO):} RF-BO bypasses this curvature hazard entirely. By updating along $\Delta \alpha \propto -h$, it aligns with the vector field's natural flow. Even though the flow splits at the saddle, the residual vector consistently points towards the attractor basin. With the aid of learning rate annealing ($\eta_t \propto 1/t$), RF-BO effectively filters the stochastic noise and executes a smooth, deterministic-like escape to the global optimum.
\end{itemize}

\subsection{Robustness in Non-Monotone Fields: The Vortex Escape Challenge}
\label{app:vortex_escape}

To further demonstrate the structural stability and universality of \textbf{RF-BO}, we introduce the \textbf{Vortex Escape} challenge. This synthetic environment is designed to simulate high-dimensional equilibrium-seeking in non-conservative vector fields, where the residual map $h(\alpha)$ is not the gradient of any scalar potential. Such non-monotone dynamics are characteristic of the competitive landscapes found in Generative Adversarial Networks (GANs), multi-agent zero-sum games, and maximum-entropy reinforcement learning. The objective is to identify a stable equilibrium at the origin $(0,0)$ starting from a high-energy initial state, while being subjected to extreme stochastic perturbations.

We evaluate three distinct algorithmic paradigms in this environment: (1) \textbf{LSE (Least Squares Estimation)}, which reformulates root-finding as minimizing the squared residual and relies on implicit Jacobian estimation; (2) \textbf{Dual-Adam}, a standard momentum-based adaptive optimizer frequently utilized for hyperparameter tuning in bilevel RL and GANs; and (3) our proposed \textbf{RF-BO}, which utilizes Jacobian-free first-order dynamics. To simulate the practical difficulty of implicit differentiation, the Jacobian estimate used by LSE is corrupted with heavy additive noise ($\sigma_J = 5.0$), while all methods face the same level of residual observation noise ($\sigma_h = 0.6$).

\begin{figure}[ht]
    \centering
    \includegraphics[width=0.85\columnwidth]{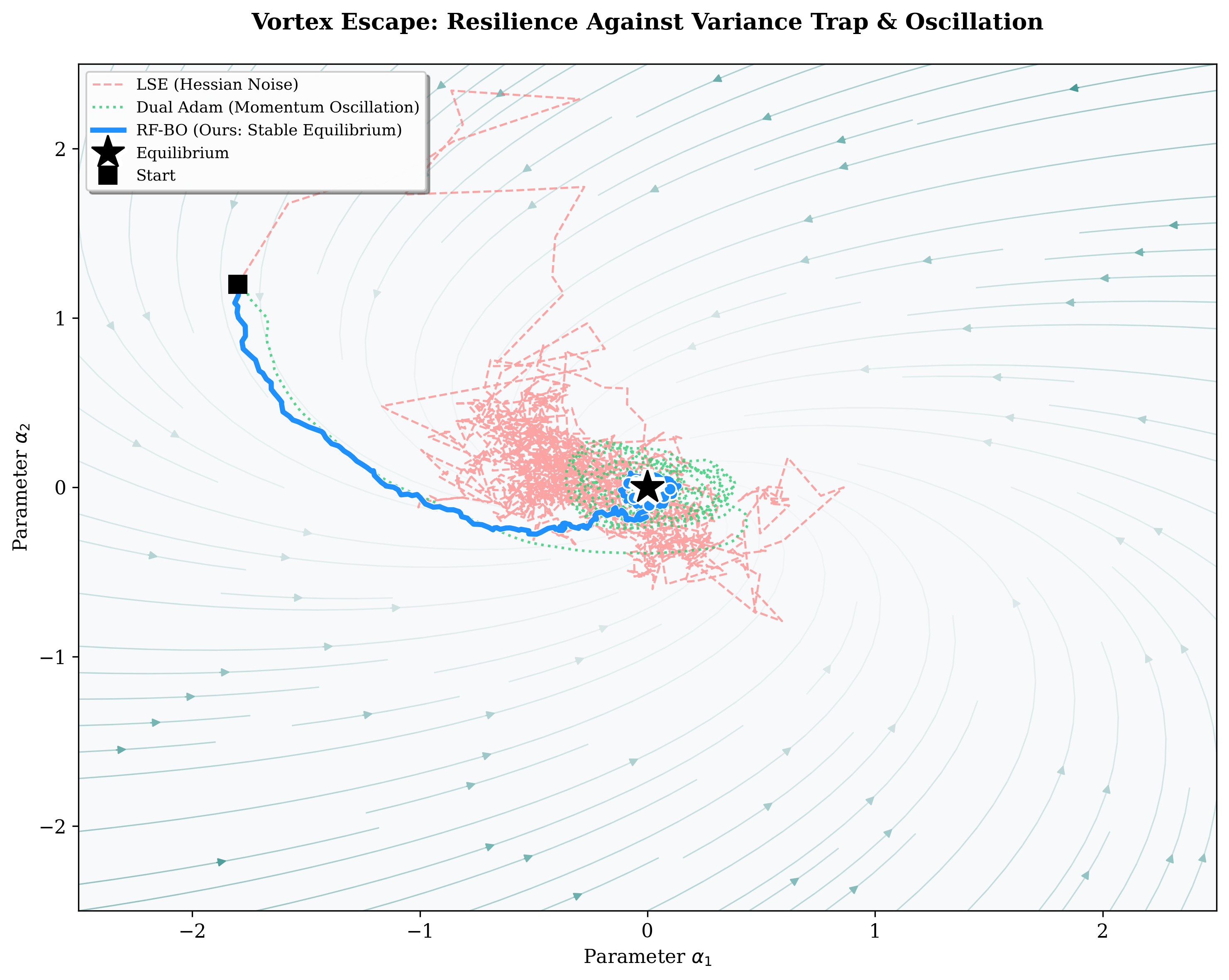} 
    \caption{\textbf{Dynamics Verification: Vortex Escape Challenge.} The background streamlines depict a rotational flow field converging to a central equilibrium. \textbf{LSE (Pink dashed)} exhibits chaotic scattering and eventually fails to converge due to the \textit{Variance Trap} in Jacobian estimation. \textbf{Dual-Adam (Green dotted)} suffers from severe \textit{Momentum Oscillation}, overshooting the attractor in the rotating field. In contrast, \textbf{RF-BO (Blue solid)} follows the first-order residual field directly, demonstrating a smooth and robust contraction to the target equilibrium.}
    \label{fig:vortex_escape}
    \vspace{-1cm}
\end{figure}

\textbf{Analysis of Results and Failure Modes.}
As visualized in Figure~\ref{fig:vortex_escape}, the trajectories reveal a decisive structural advantage for our framework. The \textbf{LSE} trajectory (Pink dashed line) exhibits extreme erraticism as it approaches the attractor. This confirms the \textit{Variance Trap} theory: the noise in the implicit Jacobian estimate ($\widehat{J} = J + \xi_J$) is effectively amplified by the residual magnitude, causing the update variance to explode. Consequently, the optimizer is trapped in a high-variance random walk and fails to commit to a stable convergence path. This failure demonstrates that squaring the residual is a fundamentally unstable strategy for stochastic bilevel tasks where Hessian estimation is imprecise.

The \textbf{Dual-Adam} baseline (Green dotted line) avoids the Variance Trap of Hessian estimation but encounters a different failure: \textit{Centrifugal Momentum Oscillation}. In rotating flow fields, the accumulation of historical gradients in the momentum buffer creates a delayed response to the changing curvature of the vector field. This causes the algorithm to repeatedly overshoot the target, leading to the large, inefficient spirals seen in the figure. This result highlights the latent risks of relying on uncalibrated momentum-based methods in equilibrium-seeking landscapes, where historical signals may become obsolete or even counter-productive.

In stark contrast, \textbf{RF-BO} (Blue solid line) demonstrates exceptional resilience. By updating directly along the first-order residual field ($\Delta \alpha \propto -h$), it avoids the noise-amplification hazards of implicit differentiation and the destabilizing inertia of high momentum. The trajectory exhibits a monotonic-like contraction toward the origin, smoothly slicing through the vortex even under heavy perturbations. This experiment establishes RF-BO as a \textbf{robust and universal solver} for root-finding bilevel optimization, capable of maintaining structural stability in complex, noisy, and non-monotone environments where standard minimization or adaptive baselines fail to reach the equilibrium.

\section{Extended Experimental Analysis and Supplementary Details}
\label{app:extended_experiments}

This appendix provides comprehensive supplementary materials for the experiments presented in Section 6. It includes detailed hyperparameter sensitivity studies, statistical significance tests, runtime complexity analysis, and full implementation specifications.

\subsection{Supplementary Details for the Synthetic Experiment}
\label{app:synthetic_experiment_details}

This section provides supplementary materials for the synthetic experiment presented in Section 6.1 of the main paper. We detail the hyperparameter sensitivity study and the comparison against the Iterative Differentiation (ITD) baseline.

\subsubsection{Hyperparameter Tuning for Baselines}
\label{app:tuning}

To ensure a rigorous comparison, in addition to the \textbf{LSE (Opt-h2)} baseline used in the main text, we conducted a separate grid search for the key hyperparameters of another strong baseline from the bilevel literature: \textbf{Iterative Differentiation (ITD)}. We focused on its most critical parameters: the unrolling depth ($K$) and the upper-level learning rate base ($\gamma_0$).

Our search space for the unrolling depth included $K \in \{3, 5, 10\}$. Larger unrolling depths degrade performance in this stochastic setting due to bias accumulation, so we select the optimal unrolling depth $K=3$ for subsequent comparisons.

Table~\ref{tab:itd_tuning} summarizes the detailed grid search results for ITD with $K=3$. The optimal performance was ultimately achieved with $\gamma_0=0.0035$, resulting in a final error of $0.0434$. This best-found configuration was used for the extended comparisons. A similar search was performed for \textbf{LSE (Opt-h2)}, identifying $\gamma_0=0.004$ as its optimal parameter.

\begin{table}[h]
\centering
\caption{Grid search results for ITD hyperparameter tuning at the optimal unrolling depth ($K=3$). Each cell reports the mean final error $|\alpha_{\text{final}} - \alpha^*|$. The best result is highlighted in bold.}
\label{tab:itd_tuning}
\begin{small}
\begin{tabular}{l|cccccccc}
\noalign{\hrule height 1.2pt}
& \multicolumn{8}{c}{\textbf{Upper-level base ($\gamma_0$)}} \\
\textbf{Depth ($K$)} & 0.0010 & 0.0020 & 0.0030 & 0.0035 & 0.0040 & 0.0050 & 0.0070 & 0.0100 \\
\midrule
3 & 0.6814 & 0.4143 & 0.1633 & \textbf{0.0434} & 0.0725 & 0.2940 & 0.6975 & 1.2157 \\
\noalign{\hrule height 1.2pt}
\end{tabular}
\end{small}
\end{table}

\subsubsection{Sensitivity Analysis and Extended Comparison}
\label{app:sensitivity}

\paragraph{Note on Fine-Grained Tuning.} 
The results presented here stem from a fine-grained hyperparameter search for RF-BO's upper-level learning rate, $\gamma_0$. This analysis validates the theoretical stability conditions discussed in Section 5.

To address the sensitivity of RF-BO to its step-size hyperparameters, we conducted an ablation study on the Synthetic Experiment testbed. We investigate the timescale separation ratio by varying the base learning rate for the upper-level update, $\gamma_0$, across several orders of magnitude while keeping the lower-level step-size schedule fixed.

The results, presented in Table~\ref{tab:sensitivity_comparison}, reveal a clear trade-off. When $\gamma_0$ is too large (e.g., $0.5$), the timescale separation is insufficient, leading to instability comparable to the `Single-Scale` baseline. As $\gamma_0$ is reduced, performance improves dramatically, peaking at an optimal value of $\gamma_0 = 0.004$. However, reducing $\gamma_0$ too far (e.g., to $0.001$) causes performance to degrade due to slow convergence within the fixed budget. This U-shaped curve empirically validates our thesis that an \textbf{optimal} timescale separation is critical.

While baseline methods like ITD plateau at an error of $0.0434$ even after extensive tuning, RF-BO demonstrates the capacity to reach a significantly lower error floor ($0.0192$) when the timescale separation is properly calibrated. Although performance degrades when deviating from this sweet spot—as expected in two-timescale dynamics—the peak performance of RF-BO represents a $\mathbf{55\%}$ error reduction compared to the strongest baseline.

\begin{table}[htb]
\centering
\caption{Sensitivity analysis of RF-BO with respect to the upper-level learning rate $\gamma_0$. The U-shaped trend confirms the importance of timescale separation. \textbf{Comparison with baselines:} At the optimal $\gamma_0=0.004$, RF-BO achieves an error of \textbf{0.0192}, significantly outperforming the best tuned \textbf{ITD} ($0.0434 \pm 0.03$) and \textbf{LSE} ($0.1630 \pm 0.03$).}
\label{tab:sensitivity_comparison}
\begin{small}
\setlength{\tabcolsep}{12pt}
\begin{tabular}{cc} 
\noalign{\hrule height 1.2pt}
\textbf{Gamma Base ($\gamma_0$)} & \textbf{RF-BO Error (Mean $\pm$ Std)} \\
\midrule
$0.01$   & $1.1227 \pm 0.0092$ \\
$0.005$  & $0.2301 \pm 0.0108$ \\
\rowcolor{blue!10} 
$0.004$  & $\mathbf{0.0192 \pm 0.0096}$ \\
$0.0035$ & $0.0927 \pm 0.0121$ \\
$0.003$  & $0.2306 \pm 0.0125$ \\
$0.001$  & $0.6979 \pm 0.1350$ \\
\noalign{\hrule height 1.2pt}
\end{tabular}
\end{small}
\end{table}

\subsection{Computational Complexity Comparison}
\label{app:complexity}

Table~\ref{tab:complexity} compares the theoretical per-iteration complexity of RF-BO against Approximate Implicit Differentiation (AID) and Iterative Differentiation (ITD), where $C_G$ and $C_H$ denote the costs of lower-level gradient and Hessian-vector products, respectively. To validate these theoretical gains, we measure the average wall-clock time per iteration on the synthetic task over 30 random seeds on an NVIDIA A100 GPU. As shown in Table~\ref{tab:runtime_comparison}, RF-BO matches the speed of the Single-Scale baseline ($0.02$ ms) while significantly outperforming ITD and LSE, further confirming its practical scalability for high-dimensional settings.
\begin{table*}[htbp]
\small
\centering
\begin{minipage}{0.45\textwidth}
\centering
\caption{Per-iteration computational complexity comparison.}
\label{tab:complexity}
\begin{tabular}{ll}
\noalign{\hrule height 1.2pt}
\textbf{Method} & \textbf{Per-Iteration Complexity} \\
\midrule
AID (Approximate) & $O(C_G + C_H \times \text{iters})$ \\
ITD & $O(K \times C_G)$ \\
\rowcolor{blue!5} \textbf{RF-BO (Ours)} & $O(C_G)$ \\
\noalign{\hrule height 1.2pt}
\end{tabular}
\end{minipage}
\hfill %
\begin{minipage}{0.52\textwidth}
\centering
\caption{Empirical runtime comparison (mean $\pm$ std over 30 seeds).}
\label{tab:runtime_comparison}
\begin{tabular}{lc}
\noalign{\hrule height 1.2pt}
\textbf{Method} & \textbf{Avg. Time per Iteration (ms)} \\
\midrule
\rowcolor{blue!5} \textbf{RF-BO (Ours)} & \textbf{0.02 $\pm$ 0.00} \\
ITD ($K=3$) & 0.16 $\pm$ 0.00 \\
LSE (Opt-h2) & 0.07 $\pm$ 0.00 \\
Single-Scale & 0.02 $\pm$ 0.00 \\
\noalign{\hrule height 1.2pt}
\end{tabular}
\end{minipage}
\end{table*}

\section{Detailed Experimental Setup and Hyperparameters}
\label{app:implementation_and_hyperparameters}
\label{app:simclr_details} 
\label{app:synthetic_details} 

All experiments were conducted on a single NVIDIA A100 GPU using the PyTorch framework. To ensure reproducibility and transparency, we provide a unified overview of the system configurations and hyperparameter settings across the three main experimental domains: Synthetic Analysis, Reinforcement Learning (SAC), and Contrastive Learning (SimCLR).

\begin{table}[H]
\centering
\caption{Unified Hyperparameter Configuration. This table aggregates the problem setup, training constraints, and optimal method-specific parameters across all experiments to facilitate direct reproducibility.}
\label{tab:unified_hyperparams}
\normalsize
\renewcommand{\arraystretch}{1.25} 
\resizebox{0.85\textwidth}{!}{ 
\begin{tabular}{l l l}
\noalign{\hrule height 1.2pt}
\textbf{Experiment Context} & \textbf{Parameter Category} & \textbf{Value / Configuration} \\
\midrule

\multicolumn{3}{l}{\cellcolor{gray!15}\textbf{1. Synthetic RF-BO (Section 6.1)}} \\
\multirow{3}{*}{Problem Setup} 
& Dimension ($\theta$) / Samples ($N$) & $10$ / $1000$ \\
& Target Constant ($C$) & $5.0$ \\
& Regularization & $\lambda=0.1$ (L2), $\kappa=0.001$ (Quartic) \\
\cmidrule(l){2-3}
\multirow{3}{*}{Training Details} 
& Iterations / Batch Size & $3000$ / $128$ \\
& Step Schedules & Lower $\eta_t \propto (t+10)^{-0.5}$, Upper $\gamma_t \propto (t+10)^{-0.6}$ \\
& Random Seeds & \textbf{15 seeds} (e.g., 80, 93, 1, 11, 3, ...) \\
\cmidrule(l){2-3}
\multirow{3}{*}{Optimal Methods} 
& \textbf{RF-BO (Ours)} & $\gamma_0 = 0.004$ \\
& ITD (Strong Baseline) & $\gamma_0 = 0.0035$, Unrolling Depth $K=3$ \\
& LSE (Opt-h2) & $\gamma_0 = 0.004$ \\
\midrule
\multicolumn{3}{l}{\cellcolor{gray!15}\textbf{2. SAC Temperature Tuning (Section 6.2)}} \\
\multirow{2}{*}{Environment} 
& Task / Horizon & Pendulum-v1 / $30,000$ timesteps \\
& Target Entropy & $-1.0$ \\
\cmidrule(l){2-3}
\multirow{2}{*}{Training Details} 
& Batch Size / Architecture & $128$ / Actor-Critic (2 hidden layers, 256 units) \\
& Random Seeds & $\{44, 47, 49, 50, 52\}$ (\textbf{5 seeds}) \\
\cmidrule(l){2-3}
\multirow{3}{*}{Optimal Methods} 
& Base Learning Rates & Actor/Critic: $3 \times 10^{-4}$; step sizes follow polynomial decay with $\gamma_t/\eta_t \to 0$ \\
& \textbf{RF-BO(Ours)} & Upper LR $\gamma_0 = 1 \times 10^{-3}$ \\
& Original-SAC & $\alpha$-LR $3 \times 10^{-4}$ \\
\midrule
\multicolumn{3}{l}{\cellcolor{gray!15}\textbf{3. SimCLR Contrastive Tuning (Section 6.3)}} \\
\multirow{2}{*}{Setup} 
& Dataset / Model & CIFAR-10 / ResNet-18 (Proj. Dim: 128) \\
& Similarity Target & $0.6$ \\
\cmidrule(l){2-3}
\multirow{2}{*}{Training Details} 
& Epochs / Batch Size & $50$ / $256$ \\
& Random Seeds & $\{42, 43, 44\}$ (\textbf{3 seeds}) \\
\cmidrule(l){2-3}
\multirow{3}{*}{Optimal Methods} 
& \textbf{RF-BO(Ours)} & $\gamma_0 = 1 \times 10^{-3}$ (Logarithmic decay) \\
& Original-Adaptive & $LR = 1 \times 10^{-5}$ \\
& Optimization & Adam ($LR=5 \times 10^{-4}$, Decay=$10^{-4}$) \\
\noalign{\hrule height 1.2pt}
\end{tabular}
}
\end{table}

\section{Code Availability and Reproducibility}
\label{app:code_link}
To support the reproducibility of our empirical results, we provide the complete source code, including the implementation of RF-BO, baseline comparisons (TTSA, BiSLS, MA-SOBA, AccBO), and all task-specific scripts. The source code has been uploaded to an anonymous GitHub repository, with the access link provided in the \texttt{code\_ICML\_RFBO1.txt} file within the supplementary materials. Furthermore, to facilitate a comprehensive review of all experimental procedures, the complete code is also provided in the \texttt{code\_ICML\_RFBO1.ipynb} file in the supplementary materials, systematically categorized and organized by experiment name.

\newpage
\section*{NeurIPS Paper Checklist}

\begin{enumerate}

\item {\bf Claims}
    \item[] Question: Do the main claims made in the abstract and introduction accurately reflect the paper's contributions and scope?
    \item[] Answer: \answerYes{} 
    \item[] Justification: The theoretical claims are explicitly supported by Section 5, and the empirical claims are validated in Section 6.

\item {\bf Limitations}
    \item[] Question: Does the paper discuss the limitations of the work performed by the authors?
    \item[] Answer: \answerYes{}
    \item[] Justification: We discussed the limitations and future work, including scaling to over-parameterized regimes like LLMs, in Section 7 (Conclusion).

\item {\bf Theory assumptions and proofs}
    \item[] Question: For each theoretical result, does the paper provide the full set of assumptions and a complete (and correct) proof?
    \item[] Answer: \answerYes{}
    \item[] Justification: All assumptions are formally stated in Section 5.1, and complete proofs are provided in Appendix B.

    \item {\bf Experimental result reproducibility}
    \item[] Question: Does the paper fully disclose all the information needed to reproduce the main experimental results of the paper to the extent that it affects the main claims and/or conclusions of the paper (regardless of whether the code and data are provided or not)?
    \item[] Answer: \answerYes{}
    \item[] Justification: Detailed hyperparameter configurations, network architectures, and training setups are systematically documented in Appendix C.

\item {\bf Open access to data and code}
    \item[] Question: Does the paper provide open access to the data and code, with sufficient instructions to faithfully reproduce the main experimental results, as described in supplemental material?
    \item[] Answer: \answerYes{}
    \item[] Justification: We have provided an anonymous GitHub repository link to our source code and scripts in Appendix D to support reproducibility.

\item {\bf Experimental setting/details}
    \item[] Question: Does the paper specify all the training and test details (e.g., data splits, hyperparameters, how they were chosen, type of optimizer) necessary to understand the results?
    \item[] Answer: \answerYes{}
    \item[] Justification: All relevant experimental settings, including grid search details and optimizer choices, are documented in Section 6 and comprehensively in Appendix C.

\item {\bf Experiment statistical significance}
    \item[] Question: Does the paper report error bars suitably and correctly defined or other appropriate information about the statistical significance of the experiments?
    \item[] Answer: \answerYes{}
    \item[] Justification: We report mean and standard deviations across multiple random seeds (e.g., 15 for synthetic, 5 for SAC).

\item {\bf Experiments compute resources}
    \item[] Question: For each experiment, does the paper provide sufficient information on the computer resources (type of compute workers, memory, time of execution) needed to reproduce the experiments?
    \item[] Answer: \answerYes{}
    \item[] Justification: The specific compute resources (NVIDIA A100 GPU) and empirical runtime analysis per iteration are disclosed in Appendix C.3.
    
\item {\bf Code of ethics}
    \item[] Question: Does the research conducted in the paper conform, in every respect, with the NeurIPS Code of Ethics \url{https://neurips.cc/public/EthicsGuidelines}?
    \item[] Answer: \answerYes{}
    \item[] Justification: This research focuses on fundamental optimization algorithms and strictly conforms to the NeurIPS Code of Ethics.

\item {\bf Broader impacts}
    \item[] Question: Does the paper discuss both potential positive societal impacts and negative societal impacts of the work performed?
    \item[] Answer: \answerNA{}
    \item[] 
    
    Justification: This paper develops a generic theoretical optimization algorithm and does not focus on applications with direct positive or negative societal impacts.
    
\item {\bf Safeguards}
    \item[] Question: Does the paper describe safeguards that have been put in place for responsible release of data or models that have a high risk for misuse (e.g., pre-trained language models, image generators, or scraped datasets)?
    \item[] Answer: \answerNA{}
    \item[] Justification: The paper does not release large-scale datasets or models with a high risk for misuse.

\item {\bf Licenses for existing assets}
    \item[] Question: Are the creators or original owners of assets (e.g., code, data, models), used in the paper, properly credited and are the license and terms of use explicitly mentioned and properly respected?
    \item[] Answer: \answerNA{}
    \item[] Justification: We strictly use standard public environments (e.g., OpenAI Gym) and datasets (CIFAR-10) following their original, well-known public licenses without modification.

\item {\bf New assets}
    \item[] Question: Are new assets introduced in the paper well documented and is the documentation provided alongside the assets?
    \item[] Answer: \answerNA{}
    \item[] Justification: The paper does not release new datasets or similar assets.

\item {\bf Crowdsourcing and research with human subjects}
    \item[] Question: For crowdsourcing experiments and research with human subjects, does the paper include the full text of instructions given to participants and screenshots, if applicable, as well as details about compensation (if any)? 
    \item[] Answer: \answerNA{}
    \item[] Justification: The research involves synthetic, simulation, and standard dataset benchmarks and does not involve human subjects or crowdsourcing.

\item {\bf Institutional review board (IRB) approvals or equivalent for research with human subjects}
    \item[] Question: Does the paper describe potential risks incurred by study participants, whether such risks were disclosed to the subjects, and whether Institutional Review Board (IRB) approvals (or an equivalent approval/review based on the requirements of your country or institution) were obtained?
    \item[] Answer: \answerNA{}
    \item[] Justification: The research does not involve human subjects, hence IRB approval is not applicable.

\item {\bf Declaration of LLM usage}
    \item[] Question: Does the paper describe the usage of LLMs if it is an important, original, or non-standard component of the core methods in this research? Note that if the LLM is used only for writing, editing, or formatting purposes and does \emph{not} impact the core methodology, scientific rigor, or originality of the research, declaration is not required.
    \item[] Answer: \answerNA{}
    \item[] Justification: LLMs were not used as a core component of the methodology. Their use was strictly limited to text editing and formatting assistance.

\end{enumerate}
\end{document}